\documentclass{article}

\usepackage{microtype}
\usepackage{graphicx}
\usepackage{booktabs}
\usepackage{hyperref}
\usepackage{graphicx,dsfont,amssymb,amsmath,cite,calc,tabulary}
\usepackage{xcolor}
\usepackage[english]{babel}
\usepackage[utf8]{inputenc}
\usepackage{algorithmic}
\usepackage{algorithm}

\usepackage{amsmath}
\usepackage{amssymb}
\usepackage{amsthm}
\usepackage{fixmath}
\usepackage{xcolor}
\usepackage{subcaption}
\usepackage{natbib}
\newtheorem{theorem}{Theorem}
\newtheorem{corollary}{Corollary}[theorem]
\newtheorem{lemma}[theorem]{Lemma}
\input{mysymbol.sty}

\usepackage{hyperref}

\begin{document}
	\title{Online Learning with Uncertain Feedback Graphs}
\author{Pouya~M.~Ghari, Yanning Shen
\thanks{P.~M.~Ghari and Y.~Shen are with the Department of Electrical Engineering and Computer Science, University of California, Irvine, CA, USA. Email:  pmollaeb@uci.edu and yannings@uci.edu}}
\maketitle
	\begin{abstract}
	Online learning with expert advice is widely used in various machine learning tasks. It considers the problem where a learner chooses one from a set of experts to take advice and make a decision. In many learning problems,  experts may be related, henceforth the learner can observe the losses associated with a subset of experts that are related to the chosen one. In this context, the relationship among experts can be captured by a feedback graph, which can be used to assist the learner's decision making. However, in practice, the nominal feedback graph often entails uncertainties, which renders it impossible to reveal the actual relationship among experts. To cope with this challenge, the present work studies various cases of potential uncertainties, and develops novel online learning algorithms to deal with uncertainties while making use of the uncertain feedback graph. The proposed algorithms are proved to enjoy sublinear regret under mild conditions. Experiments on real datasets are presented to demonstrate the effectiveness of the novel algorithms.  
	\end{abstract}

	\section{Introduction}
	In general online learning framework, there exists a learner and a set of experts, where the learner interacts with the experts to make a decision \citep{Cesa-Bianchi2006}. At each time instant, the learner chooses one of the experts and it takes the action advised by the chosen expert, then incurs the loss associated with the taken action. 
%
%
	 Conventional online learning literature mostly focuses on two settings, \emph{full information} setting \citep{LITTLESTONE1994,Cesa-Bianchi1997,Hazan2007, Resler2019} or \emph{bandit} setting \citep{Resler2019, Auer2003}. In  the full information setting, at each time instant, the learner can observe the loss associated with all experts. By contrast, in the bandit setting, the learner can only observe the loss associated with the chosen expert. However, in some applications such as the web advertising problem (where a user clicks  on an ad and it reveals information about other related ads), 
	 { the learner can make partial observations of losses associated with a subset of experts}. To cope with this scenario,  \emph{online learning with feedback graphs}  was first developed by \citet{Mannor2011}. In this context, the partial observations of losses are modeled using a directed \emph{feedback graph}, where each node represents an expert, and there exist an edge from node $i$ to node $j$, if the learner can observe the loss associated with expert $j$ while choosing expert $i$. The observations of losses associated with other experts are called learner's side observations. The full information and the bandit settings are both special cases of online learning with either a fully connected feedback graph or a feedback graph with only self loops. 
	
    Most of existing works rely on the assumption that the learner knows the feedback graph \emph{perfectly} before decision making \citep{Alon2015, Alon2017, Liu2018, Cortes2019, Arora2019}, or after decision making \citep{Alon2017, kocak2014, kocak2016noisy, Rangi2019, Cortes2020}.
	However, such information may not be available in practice. 
In addition, due to possible uncertainty of the environment, the feedback graph may be uncertain. For instance, consider the web advertising example, where there are two \emph{possibly related} ads, and the learner wants to choose and advertise one on social media. Certain group of users are interested in both ads, whereas for other users, even if they show interest in one ad, it does not indicate the same interest in the other one. However, it is not known which users will see the ad before advertisement. Therefore,  the relation between the two ads is not clear and renders the resulting feedback graph uncertain. As another example, consider an online clothing store that offers discount on an item for new customers. Suppose there are two brands A and B producing similar shirts at comparable price. The store has small and medium sizes of brand A and medium and large sizes shirts of brand B in stock.  Assuming that the store offers discount on brand B. If the user accepts the offer, and buys a medium size shirt of brand B, it implies the user is also interested in shirts of brand A. Moreover, if the user buys a large size of shirt B, this indicates no interest in shirts of brand A. Otherwise, if the user declines the offer of brand B, it only shows the user is not interested in shirts of brand B but no information is available about the preference of the user on the shirts of brand A.
%
Considering the case where the exact feedback graph may not be available, \citet{Cohen2016} shows that not knowing the entire feedback graph  can make the side observations useless and the learner may simply ignore them.
\citet{kocak2016} studies the case where the exact feedback graph is unknown but is known to be generated from the Erd\"{o}s-R\'{e}nyi model. However, such assumption may not be valid in practice. In addition, both \citet{Cohen2016} and \citet{kocak2016}  assume that the loss associated with the chosen expert is guaranteed to be observed. %

	The present paper extensively studies the case where the learner only has access to a  feedback graph that may contain uncertainties, namely \emph{nominal feedback graph}, and the learner may not be able to observe the loss associated with the chosen expert.  
	The learner relies on this nominal feedback graph to choose among experts, and then incurs a loss associated with the chosen expert. At the same time, it observes the loss associated with a subset of experts resulting from the unknown actual feedback graph. Furthermore, different from \citet{Cohen2016} and \citet{kocak2016}, the present work does not assume it is guaranteed that the learner observes the loss associated with the chosen expert. This is true in some learning tasks, e.g., apple tasting problem \citep{Helmbold2000}, such that the learner may not be aware of how much loss it incurs. In the apple tasting problem, the learner examines apples to identify rotten ones. The learner can either discard the apple or send it to the market. An apple is tasted before it is discarded. The learner incurs an unit loss if a good apple is discarded or  a rotten one is sent to the market. However, when the learner sends an apple to the market, it does not taste it. In this case, the learner is not aware of the loss of its decision after sending apples to the market. The present work studies various cases of potential uncertainties, and develops novel online learning algorithms to cope with different uncertainties in the nominal feedback graph. Regret analysis is carried out and it is proved that our novel algorithms can achieve sublinear regret under mild conditions.  Experiments on a number of real datasets are presented to showcase the effectiveness of our novel algorithms. 


	\section{Problem Statement}
	\label{prob_set}
	
    Consider the case where there exist $K$ experts and the learner chooses to take the advice of one of the experts at each time instant $t$. Let $\mathcal{G}_t=(\mathcal{V},\mathcal{E}_t)$ represent the directed nominal feedback graph at time $t$ with a set of vertices $\mathcal{V}$, where the vertex ${v}_{i} \in \mathcal{V}$ represents the $i$-th expert, and there exist an edge from $v_i$ to $v_j$ (i.e. $(i,j) \in \mathcal{E}_t$), if the learner observes the loss associated with the $j$-th expert (i.e. $\ell_t(v_j)$) with probability  $p_{ij}$ while choosing the $i$-th expert. Let $\mathcal{N}_{i,t}^{\text{in}}$ and $\mathcal{N}_{i,t}^{\text{out}}$ represent in-neighborhood and out-neighborhood of $v_i$ in $\mathcal{G}_t$, respectively. Thus, ${v}_{j} \in \mathcal{N}_{i,t}^{\text{out}}$ if there is an edge from $v_i$ to $v_j$ at time $t$ (i.e. $(i,j) \in \mathcal{E}_t$). Similarly, ${v}_{j} \in \mathcal{N}_{i,t}^{\text{in}}$ if there is an edge from $v_j$ to $v_i$ at time $t$ (i.e. $(j,i) \in \mathcal{E}_t$).
    The present paper considers non-stochastic adversarial online learning problems. At each time instant $t$, the environment privately selects a loss function $\ell_t(.)$ with $\ell_t(.):\mathcal{V} \rightarrow [0,1]$, and the nominal feedback graph $\mathcal{G}_t$ is revealed to the learner before decision making. The learner then chooses one of the experts to take its advice. Then, the learner will incur the loss associated with the chosen expert.
    Let $I_t$ denote the index of the chosen expert. Note that the learner observes $\ell_t(v_{I_t})$ with probability of $p_{I_t I_t}$, hence the loss remains unknown  with the probability of $1-p_{I_t I_t}$. 
 
	
	
	The present paper discusses different potential uncertainties in the feedback graphs, and develops novel algorithms for online learning with uncertain feedback graph. Specifically, two cases are discussed: i) \emph{online learning with informative probabilistic feedback graph}: where the probability $p_{ij}$ associated with each edge is given along with the nominal feedback graph $\mathcal{G}_t$; and ii) \emph{online learning with uninformative probabilistic feedback graph}: where only the nominal feedback graph $\mathcal{G}_t$ is revealed, but not the probabilities. 

	\section{Online Learning with Informative Probabilistic Feedback Graphs} \label{sec:IP}

	First consider the case where $\{p_{ij}\}$ are given along with the ${\cal G}_t$. This can be the case in various applications. For instance, consider a network of agents in a wireless sensor network that cooperate with each other on certain tasks such as environmental monitoring. Online learning algorithms distributed over spatial locations have been employed in climate informatics field \citep{Cesa-Bianchi2020,McQuade2012}. Assume that each agent in the network  keeps updating its local model, and there is a central unit (learner) wishes to perform a learning task based on  models and data samples distributed among agents. In this case, the  agents in the network can be viewed as experts. Consider the case where the learner chooses one of the experts and sends a request for the corresponding expert advice through a wireless link. Subset of experts which receive the request, send their advice to the learner. However, due to uncertainty in the environment or power limitation, some of the agents in the network including the chosen one may not detect the request. Therefore, the learner can only observe the advice of subset of agents in the network which detect its request. In this case, the learner can model probable advice that it can receive from experts with a nominal feedback graph. If learner knows the characteristics of the environment which is true in many wireless communication applications, the  probabilities associated with edges in the nominal feedback graph is revealed. 
	
	At each time instant $t$, upon selecting an expert and observing the losses of a subset of experts, the weights $\{w_{i,t}\}_{i=1}^K$  which indicate the reliability of experts can be updated  as follows
	\begin{align}
	w_{i,t+1} = w_{i,t} \exp\left(- \eta \hat{\ell}_t(v_i) \right),  ~~~ \forall i \in [K] \label{eq:2} 
	\end{align}
	where $[K]:=\{1,\ldots,K\}$ and $\eta$ is the learning rate. Function $\hat{\ell}_t(v_i)$ denotes the importance sampling loss estimate which can be obtained as
	\begin{align}
	\hat{\ell}_t(v_i) = \frac{\ell_t(v_i)}{q_{i,t}}\mathcal{I}(v_i \in \mathcal{S}_{t}) \label{eq:3}
	\end{align}
	where $\mathcal{S}_{t}$ represent the set of vertices associated with experts whose losses are observed by the learner at time instant $t$. The indicator function is denoted by $\mathcal{I}(.)$ and  $q_{i,t}$ is the probability that the loss $\ell_t(v_i)$ is observed. Its value depends on the algorithm, and will be specified later.

	\begin{algorithm}[t]
		\caption{Exp3-IP: Online learning with informative probabilistic feedback graph}
		\label{alg:1}
		\begin{algorithmic}
			\STATE {\bfseries Input:}{learning rate $\eta>0$.}
			\STATE \textbf{Initialize:} $w_{i,1} = 1$, $\forall i \in [K]$.
			\FOR {$t=1,\ldots,T$}
				\STATE  Observe $\mathcal{G}_t = (\mathcal{V},\mathcal{E}_t)$ and choose one of the experts according to the PMF $\pi_t$ in \eqref{eq:5}. 
				\STATE  Observe $\{\ell_t({v_i})\}_{v_i \in \mathcal{S}_t}$ and calculate loss estimate $\hat{\ell}_t(v_i)$, $\forall i \in [K]$ via \eqref{eq:3}.
				\STATE Update $w_{i,t+1}$, $\forall i \in [K]$ via \eqref{eq:2}.
			\ENDFOR
		\end{algorithmic}
	\end{algorithm}

    Let $A_t$ denote the adjacency matrix of the nominal feedback graph $\mathcal{G}_t$ with $A_t(i,j)$ denoting the $(i,j)$th entry of $A_t$. Let $X_{ij}$ be a Bernoulli random process with random variables $X_{ij}(t)=1$ with probability  $p_{ij}$. 
    When the learner chooses the $i$-th expert at time $t$, the learner observes $\ell_t(v_j)$ only if $v_j \in \mathcal{N}_{i,t}^{\text{out}}$ and $X_{ij}(t)=1$. 
    Let $F_t$ denote the number of losses observed by the learner. Due to the stochastic nature of the observations available to the learner, $F_t$ is a random variable. Furthermore, let $F_{i,t}$ denote the expected number of observed losses if the learner chooses the $i$-th expert at time $t$. Thus, we can write
	\begin{align}
	F_{i,t} \! = \mathbb{E}_t[F_t|I_t = i, A_t] =\!\!\!\!\!\! \sum_{\forall j: v_j \in \mathcal{N}_{i,t}^{\text{out}}}\!\!\!{\mathbb{E}[{X}_{ij}(t)]}  = \!\!\!\!\sum_{\forall j: v_j \in \mathcal{N}_{i,t}^{\text{out}}}{p_{ij}}.\nonumber 
	\end{align}
	
The learner then  chooses one expert according to the probability mass function (PMF) $\pi_t := (\pi_{1,t},\ldots,\pi_{K,t})$ with
	\begin{align}
	\pi_{i,t} = (1-\eta) \frac{w_{i,t}}{W_t} + \eta \frac{F_{i,t}}{\sum_{j \in \mathcal{D}_t}{F_{j,t}}}\mathcal{I}(v_i \in \mathcal{D}_t) \label{eq:5} 
	\end{align}
	where $W_t := \sum_{i=1}^{K}{w_{i,t}}$. It can be observed from \eqref{eq:5} that $\eta$ controls the trade-off between exploitation and exploration. With a smaller $\eta$,  more emphasis is placed on the first term which promotes exploitation, and the learner tends to choose the expert with larger $w_{i,t}$. The second term  allows the learner to  select experts in the dominating set $\mathcal{D}_t$ with certain probability independent of their performance in previous rounds.
	Based on \eqref{eq:5}, $q_{i,t}$ in \eqref{eq:3} can be computed as
	\begin{align}
	q_{i,t} = \sum_{\forall j: v_j \in \mathcal{N}_{i,t}^{\text{in}}}{\pi_{j,t}p_{ji}}. \label{eq:6}
	\end{align}
	The overall algorithm for online learning with uncertain feedback graph in the informative probabilistic setting,  termed Exp3-IP, is summarized in Algorithm \ref{alg:1}.
 In order to analyze the performance of Algorithm \ref{alg:1}, as well as the ensuing algorithms, we first preset two assumptions needed:\\
	    \textbf{(a1)} $0 \le \ell_t(v_i) \le 1$, $\forall t: t \in \{1,\ldots,T\}, \forall i: i \in \{1,\ldots,K\} $.\\
	    \textbf{(a2)} If $(i,j) \in \mathcal{E}_t$, the learner can observe the loss associated with the $j$-th expert with probability at least $\epsilon > 0$ when it chooses the $i$-th expert, and $(i,i) \in \mathcal{E}_t, ~ \forall i$. 

Note that (a1) is a general assumption in online learning literature e.g., \citep{Alon2015}. And (a2) assumes a nonzero probability of observing (but not guaranteed observation of)
the loss associated with the chosen expert  $\ell_t(v_{I_t})$. 
		The following theorem presents the regret bound for Exp3-IP. 
	\begin{theorem} \label{th:1}
		Under (a1), the expected regret of Exp3-IP can be bounded by
		\begin{align}
		& \sum_{t=1}^{T}{\mathbb{E}_t[\ell_t(v_{I_t})]}-\min_{v_i \in \mathcal{V}}{ \sum_{t=1}^{T}{\ell_t(v_i)} } \nonumber \\ \le & \frac{\ln K}{\eta} + \eta (1 - \frac{\eta}{2})T + \frac{\eta}{2} \sum_{t=1}^{T}{ \sum_{i=1}^{K}{\frac{\pi_{i,t}}{q_{i,t}}} }. \label{eq:17}
		\end{align}
	\end{theorem}
Proof of Theorem \ref{th:1} is included in Appendix \ref{A:1}. 	It can be seen from Theorem \ref{th:1} that the value of $\pi_{i,t}/{q_{i,t}}$ plays an important role in regret bound. Building upon Theorem \ref{th:1}, the ensuing Corollary further explores under which circumstances Exp3-IP can achieve sub-linear regret bound.
	\begin{lemma} \label{lem:4}
		Let the doubling trick (see e.g. \citet{Alon2017}) is employed to determine the value of $\eta$ and greedy set cover algorithm (see e.g. \citep{Chvatal1979}) is exploited to derive a dominating set $\mathcal{D}_t$ for the nominal feedback graph $\mathcal{G}_t$. Under (a1) and (a2), the expected regret of Exp3-IP satisfies 
		\begin{align}
		& \sum_{t=1}^{T}{\mathbb{E}_t[\ell_t(v_{I_t})]}-\min_{v_i \in \mathcal{V}}{ \sum_{t=1}^{T}{\ell_t(v_i)} } \nonumber \\ \le & \mathcal{O}\left( \sqrt{\ln K\ln(\frac{K}{\epsilon}T)\sum_{t=1}^{T}{\frac{\alpha(\mathcal{G}_t)}{\epsilon}}} + \ln (\frac{K}{\epsilon}T)\right) \label{eq:18}
		\end{align}
		where $\alpha(\mathcal{G}_t)$ denote the independence number of the nominal feedback graph $\mathcal{G}_t$.
	\end{lemma}
Proof of Lemma \ref{lem:4} is included in Appendix \ref{B}. If the learner does not know the time horizon $T$ before start decision making, doubling trick can be exploited to determine $\eta$. At time instant $t$, as long as
\begin{align}
\sum_{\tau=1}^{t}{(1 + \frac{1}{2}\sum_{i=1}^{K}{\frac{\pi_{i,\tau}}{q_{i,\tau}}})} \le 2^r \label{eq:39}
\end{align}
holds true,  Exp3-IP employs learning rate $\eta=\sqrt{\frac{\ln K}{2^{r+1}}}$, where $r\geq 0$ is the smallest integer that can satisfy the inequality in \eqref{eq:39}.  
According to \eqref{eq:18}, Exp3-IP can achieve sub-linear regret. Furthermore,  \eqref{eq:18} shows that the regret bound of Exp3-IP depends on $\frac{1}{\epsilon}$. Larger $\epsilon$ indicates that the learner is less uncertain about the nominal feedback graph. In other words higher confidence of the nominal feedback graph leads to tighter regret bound.

	\section{Online Learning with Uninformative Probabilistic Feedback Graphs}
	
	The previous section deals with the case where probabilities associated with edges of $\mathcal{G}_t$ are revealed. In this section, we will  study the scenario where the nominal feedback graph $\mathcal{G}_t$ is static and is revealed to the learner while the probabilities  $\{p_{ij}\}$ associated with edges are not given, which is called \emph{uninformative probabilistic feedback graph}. In this section the nominal feedback graph is denoted by $\mathcal{G} = (\mathcal{V},\mathcal{E})$. In this case, estimates of probabilities $\{p_{ij}\}$ can be updated and employed to assist the learner with future decision making. For example, consider the problem of online advertisement, where a website is trying to decide which product to be advertised via online survey with a multiple choice question. Specifically, users are asked whether they are interested in certain product along with possible reasons (cost, color, etc). Note that the answer to certain product may also indicate the participant's potential interest in other products with similar cost or color.  For instance, if the participant indicates that he or she is interested in the product because of its affordable cost, this implies \emph{potential} interest in other products with the same or lower price.  In this case, the relationship among products can be modeled by a nominal feedback graph, where an edge exists between two nodes (products) if they share same or similar attributes (cost, color), which implies that users \emph{may be} interested in both products. Such nominal feedback graph can then be used to assist the website to make a decision on which product to advertise .
	 However, the actual relationship between the the user's interests in the products remains uncertain, which leads to uncertainty in the nominal feedback graph. Since  attributes (cost, color, etc) of products do not change over time, the nominal feedback graph is static, while the probabilities associated with edges in the nominal feedback graph are unknown. Faced with this practical challenge, two approaches will be developed in this section, to estimate either the unknown probability or the importance sampling loss in \eqref{eq:3}, which will then be employed to assist the learner's decision making. 

\subsection{Estimation-based Approach}
	 In the present subsection, we will further explore the general scenario where the value of $p_{ij}$ may vary across edges, while the nominal feedback graph $\mathcal{G}_{t}$ is static. 
	 Since ${X}_{ij}$ defined under \eqref{eq:3} is a mean ergodic random process \citep{Popoulis2002} in this scenario, the sample mean of $\{{X}_{ij}(t)\}$ converges to $p_{ij}$, i.e., the expected value of ${X}_{ij}(t)$. Let $\mathcal{T}_{ij,t}$ represent a set collecting time instants before $t$ when the learner chooses to take the advice of the $i$-th expert and there is an edge between $v_i$ and $v_j$ in the nominal feedback graph $\mathcal{G}$. In other word, $\mathcal{T}_{ij,t}$ can be defined as
	 \begin{align}
	 \mathcal{T}_{ij,t} = \{\tau|A_{\tau}(i,j)=1, I_{\tau} = i, 0<\tau< t\}. \label{eq:7}
	 \end{align}
	 Based on the above discussion, $p_{ij}$ can be estimated as 
	 \begin{align}
	 \hat{p}_{ij,t} = \frac{1}{C_{ij,t}}\sum_{\tau \in \mathcal{T}_{ij,t}}{{X}_{ij}(\tau)} \label{eq:8}
	 \end{align}
	 where $C_{ij,t}:=|\mathcal{T}_{ij,t}|$ is the cardinality of  $\mathcal{T}_{ij,t}$. Since ${X}_{ij}$ is a mean ergodic Bernoulli random process, $\hat{p}_{ij,t}$ is an unbiased maximum likelihood (ML) estimator of ${p}_{ij}$.  
	 
	 Note that a sufficient number of observations of the random process $X_{ij}$ is needed, in order to provide a reliable estimation in \eqref{eq:8}. To this end, the learner performs exploration in the first $KM$ time instants to ensure that $C_{ij,t} \ge M$, $\forall (i,j) \in \mathcal{E}_{t}$, where the value of $M$ is determined by the learner. Specifically, in the first $KM$ time instants, the learner chooses all experts in $\mathcal{V}$, one by one $M$ times, i.e. the learner selects expert $v_k$, with $k = t - \left \lfloor{\frac{t}{K}} \right \rfloor K$ when $t \le KM$. For $t>KM$, the learner draws one of the experts according to the following PMF
	 \begin{align}
	 \pi_{i,t} = (1-\eta) \frac{w_{i,t}}{W_t} + \frac{\eta}{|\mathcal{D}|}\mathcal{I}(v_i \in \mathcal{D}), \forall i \in [K] \label{eq:10}
	 \end{align}
	 where $\mathcal{D}$ denotes a dominating set for  the nominal feedback graph $\mathcal{G}$. In order to  obtain a reliable loss estimate to assist the learner's decision making, we will approximate the importance sampling loss estimate in \eqref{eq:3} using the estimated probability $\hat{p}_{ij,t}$. 
	 In this context, the probability of observing $\ell_t(v_i)$ can be approximated as
	 \begin{align}
	 \hat{q}_{i,t} = \sum_{\forall j: v_j \in \mathcal{N}_{i,t}^{\text{in}}}{\pi_{j,t}(\hat{p}_{ji,t}+\frac{\xi}{\sqrt{M}})} \label{eq:11} 
	 \end{align}
	 where $\xi \ge 1$ is a parameter selected by the learner. Consequently, the importance sampling loss estimates can be obtained as
	 \begin{align}
	 \tilde{\ell}_t(v_i) = \frac{\ell_t(v_i)}{\hat{q}_{i,t}}\mathcal{I}(v_i \in \mathcal{S}_t). \label{eq:12}
	 \end{align}
	 With the estimates in hand, the  weights $\{w_{i,t}\}_{i=1}^{K}$ can be updated as follows 
	 \begin{align}
	 w_{i,t+1} = w_{i,t}\exp\left(- \eta \tilde{\ell}_t(v_i)\right),~~ \forall i \in [K]. \label{eq:13}
	 \end{align}
	 
	 The procedure that the learner chooses among experts when the probabilities are unknown is presented in Algorithm \ref{alg:2},  named  Exp3-UP.
	 \begin{algorithm}[t]
	 	\caption{Exp3-UP: Online learning with uninformative probabilistic feedback graphs}
	 	\label{alg:2}
	 	\begin{algorithmic}
	 		\STATE {\bfseries Input:}{ learning rate ${\eta}>0$, the minimum number of observations $M$, $\mathcal{G} = (\mathcal{V},\mathcal{E})$. }
	 		\STATE \textbf{Initialize:} $w_{i,1} = 1$, $\forall i \in [K]$, $\hat{p}_{ij,1}=0$, $\forall (i,j) \!\in\! \mathcal{E}$.
	 		\FOR {$t=1,\ldots,T$}
	 		\IF{$t \le KM$} 
	 		\STATE Set $k \!=\! t \!-\! \lfloor\frac{t}{K}\rfloor K$ and draw the expert node $v_k$. 
	 		\ELSE

	 		\STATE Select one of the experts according to the PMF $\pi_t=(\pi_{1,t},\ldots,\pi_{K,t})$ , with $\pi_{i,t}$ in \eqref{eq:10}.
	 		\ENDIF
	 		\STATE Observe $\{(i,\ell_t({v_i})):v_i \in \mathcal{S}_t\}$ and  compute $\tilde{\ell}_t(v_i)$, $\forall i \in [K]$ as in \eqref{eq:12}.
	 		\STATE { Update $\hat{p}_{ij,t+1}$, $\forall (i,j) \in \mathcal{E}_{t}$ via \eqref{eq:8}}.
	 		\STATE Update $w_{i,t+1}$, $\forall i \in [K]$ via \eqref{eq:13}.
	 		\ENDFOR
	 	\end{algorithmic}
	 \end{algorithm}
 The following theorem establishes the regret bound of Exp3-UP.
 	\begin{theorem} \label{th:2}
 		 Under (a1), the expected regret of Exp3-UP satisfies
 		\begin{align}
 		& \sum_{t=1}^{T}{ \mathbb{E}_{t}[\ell_{t}(v_{I_t})] } - \min_{v_i \in \mathcal{V}}\sum_{t=1}^{T}{ \ell_{t}(v_i)} \nonumber \\ \le & \frac{\ln K}{\eta} + (K-1)M + \eta(1-\frac{\eta}{2})(T-KM) + \sum_{t=KM+1}^{T}{ \sum_{i=1}^{K}{ \frac{\pi_{i,t}}{q_{i,t}}(\frac{2\xi}{\sqrt{M}}+\frac{\eta}{2}) } } \label{eq:19}
 		\end{align}
 		with probability at least
 		\begin{equation}
 		\delta_{\xi}:=\prod_{t=KM+1}^{T}{ \prod_{(i,j)\in\mathcal{E}_{t}}{ \left(1-2 \exp(-\frac{2\xi^2 C_{ij,t}}{M+4\xi \sqrt{M}}) \right) } }.\nonumber
 		\end{equation}
 	\end{theorem}
  See proof of Theorem \ref{th:2}  in Appendix \ref{E}. 
 	The following Corollary states conditions under which the regret bound in \eqref{eq:19} holds with high probability, i.e., $\delta_{\xi}=1-\mathcal{O}(\frac{1}{T})$, the proof can be found in Appendix \ref{G}.
 	\begin{corollary} \label{cor:3}
 	If $M \ge \left(\frac{4\xi \ln (KT)}{\xi^2-\ln (KT)}\right)^2$ and $\xi > \sqrt{\ln (KT)}$,  
 	under (a1) and (a2) the expected regret of Exp3-UP satisfies
 	\begin{align}
 	    & \sum_{t=1}^{T}{ \mathbb{E}_{t}[\ell_{t}(v_{I_t})] } - \min_{v_i \in \mathcal{V}}{\sum_{t=1}^{T}{ \ell_{t}(v_i)}} \le \mathcal{O}\left(\frac{\alpha(\mathcal{G})}{\epsilon}\ln(KT)\sqrt{K \ln (KT)}T^{\frac{2}{3}}\right) \label{eq:29}
 	\end{align}
 	with probability at least $1-\mathcal{O}(\frac{1}{T})$.
 	\end{corollary}
 	Note that according to Algorithm \ref{alg:2} and Corollary \ref{cor:3}, knowing the value of the time horizon $T$ is required so that the learner can choose the values for $M$ and $\xi$ to achieve the sublinear regret bound in \eqref{eq:29}, which may not be feasible, and can be resolved by resorting to doubling trick. In this case, if $2^b < t \le 2^{b+1}$ where $b\in\mathbb{N}$, the learner performs the Exp3-UP with parameters
    \begin{subequations} \label{eq:36}
    \begin{align}
        \eta & = \sqrt{\frac{\ln K}{2^{b+1}}} \label{eq:36a} \\
        M & = \left \lceil{2^{\frac{2(b+1)}{3}}\frac{1}{\sqrt{K}} + \ln 4K} \right \rceil \label{eq:36b} \\
        \xi & = \left(2K^{\frac{1}{4}}+\sqrt{4\sqrt{K}+1}\right)\sqrt{\ln (K2^{b+3})}. \label{eq:36c}
    \end{align}
    \end{subequations}
When the learner realizes that the value of $M$ needs to be increased, it then performs exploration to guarantee that at least $M$ samples of the mean ergodic random process $X_{ij}$ are observed. The following lemma shows that when doubling trick is employed, Exp3-UP can achieve sub-linear regret without knowing the time horizon beforehand, the proof of which is in Appendix \ref{H}.
 	\begin{lemma} \label{lem:5}
 	Assuming that the doubling trick is employed to determine the value of $\eta$, $M$ and $\xi$ at each time instant and the greedy set cover algorithm is utilized to obtain a dominating set $\mathcal{D}$ of the nominal feedback graph. If $T>K$, the regret of Exp3-UP satisfies
 	\begin{align}
 	    & \sum_{t=1}^{T}{ \mathbb{E}_{t}[\ell_{t}(v_{I_t})] } - \min_{v_i \in \mathcal{V}}{\sum_{t=1}^{T}{ \ell_{t}(v_i)}} \nonumber \\ \le & \mathcal{O}\left(\!\!\frac{\alpha(\mathcal{G})}{\epsilon}\ln(T)\ln(KT)\sqrt{K \ln(KT)} T^{\frac{2}{3}} \!+\! \ln T \!\!\right) \label{eq:30}
 	\end{align}
 	with probability at least $1-\mathcal{O}(\frac{1}{K})$.
 	\end{lemma}


    \subsection{ Geometric Resampling-based Approach}	 
 	Another approach to obtain a reliable loss estimate is to employ geometric resampling. Similar to Exp3-UP,  if $t \le KM$ the learner chooses the $k$-th expert at time instant $t$ where $k=t-\left \lfloor{t/K} \right \rfloor K$. In this way, it is guaranteed that at least $M$ samples of the mean ergodic random process $X_{ij}$ are observed. Based on these observations,  a loss estimate is obtained whose expected value is an approximation of the loss $\ell_t(v_i)$, $\forall i \in [K]$. At  $t > KM$, the learner draws one of the experts according to the following PMF
 	\begin{align}
 	\pi_{i,t} = (1-\eta) \frac{w_{i,t}}{W_t} + \frac{\eta}{|\mathcal{D}|}\mathcal{I}(v_i \in \mathcal{D}),~~ \forall i \in [K] \label{eq:21}
 	\end{align}
    where $\mathcal{D}$ represents a dominating set for $\mathcal{G}$. Furthermore, at each time instant $t > KM$, let $\tau_{ij,1}^{(t)}, \ldots, \tau_{ij,M}^{(t)}$ denote the last $M$ time instants before $t$ at which the learner observes samples of the random process $X_{ij}$. Let $Y_{ij,1}(t), \ldots, Y_{ij,M}(t)$ denote a random permutation of $X_{ij}(\tau_{ij,1}^{(t)}), \ldots, X_{ij}(\tau_{ij,M}^{(t)})$. At each time instant $t$, the learner draws with replacement $M$ experts according to PMF $\{\pi_{i,t}\}$ in \eqref{eq:21} in $M$ independent trials. Let $d_{u}$ denote the index of the 
   selected expert at the $u$-th trial, and $P_{i,1}(t), \ldots, P_{i,M}(t)$ be a sequence of random variables associated with $v_i$ at time instant $t$ where $P_{d_u,u}(t)=1$ and $P_{d_{u}^\prime,u}(t)=0$ if $d_{u}^\prime \neq d_u$. 
 	Let 
 	\begin{align}
 	    Z_{i,u}(t) =\!\!\!\! \sum_{\forall j: v_{j} \in \mathcal{N}_{i,t}^{\text{in}}}\!\!{P_{j,u}(t)Y_{ji,u}(t)} \label{eq:23} 
 	\end{align}
 	for all $1 \le u \le M$.
 	An under-estimate of loss  can then be obtained as
 	\begin{align}
 	    \tilde{\ell}_t(v_i) = Q_{i,t}\ell_t(v_i)\mathcal{I}(v_i \in \mathcal{S}_t). \label{eq:22}
 	\end{align}
 	where $Q_{i,t} := \min \left\{ \{u \mid 1 \le u \le M, Z_{i,u}(t)=1\} , M \right\}$, and the expected value of  $\tilde{\ell}_t(v_i)$ can be written as 
 	\begin{align}
 	    \mathbb{E}_{t}[{\tilde{\ell}_{t}(v_i)}] = \left(1-(1-q_{i,t})^M\right)\ell_{t}(v_i), \label{eq:35}
 	\end{align}
 	see \eqref{eq:39ap} -- \eqref{eq:42ap} in Appendix \ref{D} for detailed derivation. Then, the weights $\{w_{i,t}\}_{i=1}^K$ are updated as in \eqref{eq:13} using the loss estimate $\tilde{\ell}_t(v_i)$ in \eqref{eq:22}. The geometric resampling based online expert learning framework (Exp3-GR) is summarized in Algorithm \ref{alg:3}, and {its regret bound is presented in the following theorem}. 
 		\begin{theorem} \label{th:3}
 		 Under (a1) and (a2), the expected regret of Exp3-GR is bounded by
 		\begin{align}
 		& \sum_{t=1}^{T}{ \mathbb{E}_{t}[\ell_t(v_{I_t})] } - \min_{v_i \in \mathcal{V}}{ \sum_{t=1}^{T}{\ell_t(v_i)} } \nonumber \\  \le & \frac{\ln K}{\eta} + (K-1)M + \sum_{t=KM+1}^{T}{ (1-q_{i,t})^M } \nonumber \\ & + \eta(1-\eta)(T-KM) + \eta \sum_{t=KM+1}^{T}{\sum_{i=1}^{K}{ \frac{\pi_{i,t}}{q_{i,t}} } }. \label{eq:20}
 		\end{align}
 	\end{theorem}
 	
 	\begin{algorithm}[t]
	 	\caption{Exp3-GR: Online learning with geometric resampling}
	 	\label{alg:3}
	 	\begin{algorithmic}
	 		\STATE {\bfseries Input:}{learning rate ${\eta}>0$, the minimum number of observations $M$, $\mathcal{G} = (\mathcal{V},\mathcal{E})$. }
	 		\STATE \textbf{Initialize:} $w_{i,1} = 1$, $\forall i \in [K]$.
	 		\FOR {$t=1,\ldots,T$}
	 		\IF{$t \le KM$} 
	 		\STATE Set $k \!=\! t \!-\! \lfloor\frac{t}{K}\rfloor K$ and draw the expert node $v_k$. 
	 		\ELSE 
	 		\STATE Select one  expert according to PMF $\pi_t$ in \eqref{eq:21}.
	 		\STATE  Observe $\{\ell_t({v_i}):v_i \in \mathcal{S}_t\}$ and compute $\tilde{\ell}_t(v_i) $, $\forall i \in [K]$ via \eqref{eq:22}.
	 		\STATE Update $w_{i,t+1}$, $\forall i \in [K]$ via \eqref{eq:13}.
	 		\ENDIF
	 		\ENDFOR
	 	\end{algorithmic}
	 \end{algorithm}


 	The proof of Theorem \ref{th:3} is presented in Appendix \ref{D}. Building upon Theorem \ref{th:3}, the following Corollary presents the conditions under which Exp3-GR can obtain sub-linear regret.
 	\begin{corollary} \label{cor:2}
 	Assume that greedy set cover algorithm is employed to find a dominating set of the nominal feedback graph $\mathcal{G}$. If $M \ge \frac{|\mathcal{D}| \ln T}{2\eta \epsilon}$, under (a1) and (a2), Exp3-GR satisfies
 	\begin{align}
 	    & \sum_{t=1}^{T}{ \mathbb{E}_{t}[\ell_t(v_{I_t})] } - \min_{v_i \in \mathcal{V}}{ \sum_{t=1}^{T}{\ell_t(v_i)} } \nonumber \\ \le &  \mathcal{O}\left(\frac{\alpha(\mathcal{G})}{\epsilon}\sqrt{\ln K}(\ln(KT)+K \ln T)\sqrt{T}\right). \label{eq:33}
 	\end{align}
 	\end{corollary}
 	Proof of Corollary \ref{cor:2} is in Appendix \ref{F}. Achieving the sub-linear regret in \eqref{eq:33} requires that the learner knows the time horizon $T$, beforehand which may not be possible in some cases. When the learner does not know $T$, doubling trick can be utilized to achieve sub-linear regret. The following Lemma is proved in Appendix \ref{I}, shows the regret bound for Exp3-GR when doubling trick is employed to find values of $\eta$ and $M$ without knowing the time horizon $T$. In this case, at time instant $t$, when $2^b < t \le 2^{b+1}$, for $\eta$ and $M$ the following parameters are chosen for Exp3-GR
    \begin{subequations} \label{eq:37}
        \begin{align}
            \eta &= \sqrt{\frac{\ln K}{2^{b+1}}} \label{eq:37a} \\
             M &= \left \lceil{\frac{(b+1)\sqrt{2^{b-1}}|\mathcal{D}|\ln 2}{\epsilon \sqrt{\ln K}}} \right \rceil. \label{eq:37b}
        \end{align}
    \end{subequations}
    When the learner realizes that  $M$ needs to be increased, it performs exploration to guarantee that at least $M$ samples of the mean ergodic random process $X_{ij}$ are observed.
 	\begin{lemma} \label{lem:6}
 	Employing doubling trick  to select  $\eta$ and $M$ at each time instant, and supposing that a dominating set for the nominal feedback graph $\mathcal{G}$ is obtained using greedy set cover algorithm, the expected regret of Exp3-GR satisfies
 	\begin{align}
 	    & \sum_{t=1}^{T}{ \mathbb{E}_{t}[\ell_t(v_{I_t})] } - \min_{v_i \in \mathcal{V}}{ \sum_{t=1}^{T}{\ell_t(v_i)} } \nonumber \\ \le & \mathcal{O}\left(\frac{\alpha(\mathcal{G})\ln T}{\epsilon}\sqrt{\ln K}(\ln(KT)+K)\sqrt{T} \right). \label{eq:34}
 	\end{align}
 	\end{lemma}
 	Comparing Lemma \ref{lem:5} with Lemma \ref{lem:6}, it can be observed that Exp3-GR achieves a tighter regret bound  \emph{with probability 1} when the number of experts $K$ is negligible in comparison with  horizon $T$. However, note that  choosing an appropriate  $M$ for Exp3-GR requires knowing  $\epsilon$ or a lower bound of $\epsilon$, which may not be feasible in general, while such information is not required for Exp3-UP in order to guarantee the regret bound in \eqref{eq:30}.
 	Furthermore, if the number of experts $K$ is large such that $K > \mathcal{O}(T^\frac{1}{3})$, Exp3-UP can achieve tighter regret bound compared with that of Exp3-GR. For example, if $K=\mathcal{O}(\sqrt{T})$, the regret of Exp3-UP is bounded from above by
 $
 	    \frac{\alpha(\mathcal{G})}{\epsilon}\sqrt{\ln^5(T)}T^\frac{11}{12} $,
 	which is tighter than the regret bound of Exp3-GR as it is
$
 	    \frac{\alpha(\mathcal{G})}{\epsilon}\sqrt{\ln^3(T)}T
$.
 	
\textbf{Comparison with \citet{kocak2016}.} 	Note that while Exp3-GR and Exp3-Res proposed in \citet{kocak2016} both employ the geometric resampling technique, there exist two major differences: i)  Exp3-Res assumes the actual feedback graph is generated from  Erd\"{o}s-R\'{e}nyi model, and the probabilities  of the presence of edges are equal across all edges, while Exp3-GR considers {the unequally probable case}; and ii) unlike Exp3-Res,  Exp3-GR does not assume that the learner is guaranteed to observe the loss associated with the chosen expert.
 	
 	\section{Experiments}

 Performance of the proposed algorithms Exp3-IP, Exp3-UP and Exp3-GR are compared with online learning algorithms Exp3 \citep{Auer2003}, Exp3-Res \citep{kocak2016} and Exp3-DOM \citep{Alon2017}.  Exp3 considers bandit setting, and Exp3-Res assumes Erd\"{o}s-R\'{e}nyi model for the feedback graph. Furthermore, Exp3-DOM treats the nominal feedback $\mathcal{G}_{t}$ as the actual one without considering uncertainties. Performance is tested  for regression task over several  real datasets obtained from the UCI Machine Learning Repository \citep{Dua2019}: \\
 	    \textbf{Air Quality}: This dataset contains $9,358$ instances of responses from  sensors located in a polluted area, each with $13$ features. The goal is to predict polluting chemical concentration in the air \citep{DeVito2008}.\\
 	    \textbf{CCPP}: The dataset has $9,568$  samples, with $4$ features including temperature, pressure, etc, collected from a combined cycle power plant. The goal is predicting hourly electrical energy output \citep{Tufekci2014}.\\
 	    \textbf{ Twitter:}  This dataset contains $14,000$ samples  with 77 features including e.g., the length of discussion on a given topic and the number of new interactive authors. The goal is to predict average number of active discussion on a certain topic \citep{kawala2013}.\\
 	    \textbf{ Tom's Hardware}: The dataset contains $10,000$ samples from a technology forum with $96$ features. The goal is to predict the average number of display about a certain topic on Tom’s hardware  \citep{kawala2013}.

 	        \begin{table}
      \caption{MSE and  standard deviation $(\times 10^{-3})$ on Air and CCPP datasets in equally probable setting.}
      \centering
      \label{table:2}
      \begin{tabular}{lll}
        \toprule
        \cmidrule(r){2-3}
                    &Air Quality   & CCPP    \\
        \midrule
        Exp3        & $8.70 \pm 0.26$    & $20.95 \pm 0.28$      \\
        Exp3-Res    & $11.23 \pm 0.37$    & $12.86 \pm 0.23$        \\
        Exp3-DOM    & $6.40 \pm 0.26$   & $13.76 \pm 0.34$       \\
        Exp3-IP     & $\textbf{4.13} \pm 0.27$    & $\textbf{7.27} \pm \textbf{0.13}$       \\
        Exp3-UP     & $4.63 \pm 0.37$    & $8.78 \pm 0.29$        \\
        Exp3-GR     & $4.71 \pm \textbf{0.20}$    & $8.41 \pm 0.14$   \\
        \bottomrule
      \end{tabular}
    \end{table}
    \begin{table}
      \caption{Performance on Twitter and Tom's Hardware datasets in equally probable setting.}
      \centering
      \label{table:3}
      \begin{tabular}{lll}
        \toprule
        \cmidrule(r){2-3}
                     &   Twitter      &   Tom's  \\
        \midrule
        Exp3        &     $7.84 \pm 0.29$      &       $5.74 \pm 0.43$ \\
        Exp3-Res    &     $10.01 \pm 0.40$     &       $6.07 \pm 0.51$  \\
        Exp3-DOM    &     $5.20 \pm 0.22$      &       $4.77 \pm 0.45$  \\
        Exp3-IP     &     $\textbf{4.19} \pm \textbf{0.18}$  &   $\textbf{3.12} \pm 0.36$  \\
        Exp3-UP     &     $4.47 \pm 0.20$      &       $3.83 \pm 0.42$  \\
        Exp3-GR     &     $4.64 \pm 0.26$      &       $3.51 \pm \textbf{0.35}$   \\
        \bottomrule
      \end{tabular}
    \end{table}

 	In all experiments,  $9$ experts are trained using $10\%$ of each dataset. Among them, $8$  are trained via  kernel ridge regression, with $5$ using RBF kernels with bandwidth of $10^{-2},10^{-1},1,10,100$, $3$ using Laplacian kernels with bandwidth $10^{-2},1,100$, and one expert is obtained via linear regression. The nominal graph $\mathcal{G}_{t}$ is fully connected.  Performance of algorithms are evaluated based on mean square error (MSE) over $20$ independent runs, which is defined as
 	\begin{align}
 	\text{MSE}: = \frac{1}{20} \sum_{n=1}^{20}{ \frac{1}{t}\sum_{\tau=1}^{t}{ (\hat{y}_{\tau,n}-y_{\tau})^2 } }  \label{eq:38}
 	\end{align}
 	where $\hat{y}_{\tau,n}$ and $y_{\tau}$ are the prediction of the chosen expert at $n$-th run and the true label of the datum at time $\tau$, respectively. The learning rate $\eta$ is set to $\frac{1}{\sqrt{t}}$ for all algorithms except for Exp3-Res which uses the suggested learning rate by \citet{kocak2016}. Parameter $M$ is set as $25$ for both Exp3-UP and Exp3-GR and $\xi=1$ for Exp3-UP. { All experiments were carried out using Intel(R) Core(TM) i7-10510U CPU @ 1.80 GHz 2.30 GHz processor with a 64-bit {Windows} operating system.}


     We first tested the equally probable setting where probabilities $p_{ij}=0.25$, $\forall i,j$. Table \ref{table:2} lists the MSE performance  along with standard deviation of MSE for Air Quality and CCPP datasets.  Table \ref{table:3} shows the MSE performance along with its standard deviation for Twitter and Tom's Hardware datasets. It can be observed that, knowing the exact probability enables Exp3-IP to achieve the best accuracy, and our novel Exp3-UP and Exp3-GR obtain lower MSE than Exp3. Moreover, note that in this case, the actual feedback graph is indeed generated from the Erd\"{o}s-R\'{e}nyi model. It turns out that Exp3-Res built upon this assumption obtains larger MSE compared to Exp3-UP and Exp3-GR. 
    
    \begin{figure}[ht]
     \centering
         \begin{subfigure}{.5\textwidth}
             \centering
             \includegraphics[width=1\linewidth]{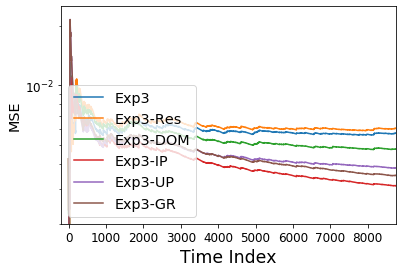}  
             \caption{Equally Probable Setting.}
             \label{fig:sub-first}
         \end{subfigure}
         \hfill
         \begin{subfigure}{.5\textwidth}
             \centering
             \includegraphics[width=1\linewidth]{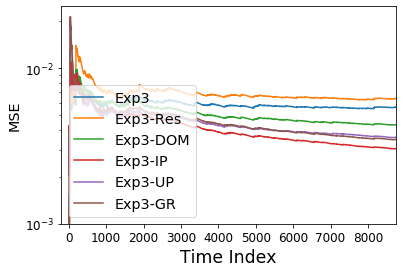}  
             \caption{Unequally Probable Setting.}
             \label{fig:sub-second}
         \end{subfigure}
         \caption{MSE performance on Tom's dataset.}
         \label{fig:1}
    \end{figure}

    %

\begin{table}
      \caption{Performance on Air Quality and CCPP datasets in the unequally probable setting}
      \centering
      \label{table:1}
      \begin{tabular}{lll}
        \toprule
        \cmidrule(r){2-3}
          &Air Quality  & CCPP \\
        \midrule
        Exp3        & $8.12 \pm 0.48$    & $20.49 \pm 0.23$  \\
        Exp3-Res    & $11.68 \pm 0.35$    & $10.12 \pm 0.24$  \\
        Exp3-DOM    & $5.60 \pm 0.36$   & $11.22 \pm 0.21$   \\
        Exp3-IP     & $\textbf{4.14} \pm \textbf{0.18}$    & $\textbf{7.19} \pm \textbf{0.18}$  \\
        Exp3-UP     & $\textbf{4.34} \pm 0.32$    & $8.08 \pm \textbf{0.14}$  \\
        Exp3-GR     & $4.68 \pm \textbf{0.17}$    & $8.42 \pm 0.23$   \\
        \bottomrule
      \end{tabular}
    \end{table}  
    \begin{table}
      \caption{Performance on Twitter and Tom's Hardware in the unequally probable setting}
      \centering
      \label{table:4}
      \begin{tabular}{lll}
        \toprule
        \cmidrule(r){2-3}
          &   Twitter      &   Tom's \\
        \midrule
        Exp3        &     $7.85 \pm 0.21$      &       $5.63 \pm 0.42$ \\
        Exp3-Res    &     $9.34 \pm 0.42$     &       $6.40 \pm 0.36$  \\
        Exp3-DOM    &     $5.64 \pm 0.21$      &       $4.33 \pm 0.34$  \\
        Exp3-IP     &     $\textbf{4.27} \pm 0.22$      &       $\textbf{3.04} \pm \textbf{0.26}$  \\
        Exp3-UP     &     $4.60 \pm 0.29$      &       $3.60 \pm 0.39$  \\
        Exp3-GR     &     $4.74 \pm \textbf{0.19}$      &       $3.48 \pm 0.32$   \\
        \bottomrule
      \end{tabular}
    \end{table}


 We further tested the unequally probable case, with $p_{ij}$ drawn from uniform distribution $\mathcal{U}[0.25,0.5]$. Tables \ref{table:1} and \ref{table:4} list the MSE of all algorithms along with standard deviation of MSE for Air Quality, CCPP, Twitter and Tom's Hardware datasets, respectively. It can be observed that Exp3-IP  obtains the best accuracy. This shows that knowing the  probabilities can indeed help  obtain better performance. Furthermore, it can be observed that Exp3-UP and Exp3-GR can achieve lower MSE in comparison with Exp3 which shows the effectiveness of using the information given by the uncertain graph. In addition, lower MSE of Exp3-UP and Exp3-GR compared to Exp3-DOM indicates that considering the uncertain graph $\mathcal{G}_t$ as a certain graph can degrade MSE. Moreover, it can be observed  Exp3-UP and Exp3-GR outperform Exp3-Res when the actual feedback graph is not generated by Erd\"{o}s-R\'{e}nyi model. It can be observed Exp3-IP achieves lower MSE than Exp3-GR and Exp3-UP, since  the learner has access to the probabilities, while Exp3-UP and Exp3-GR do not rely on such prior information. Figure \ref{fig:1} illustrates the MSE performance of algorithms on Tom's Hardware dataset over time. It can be readily observed that our prposed algorithms converge faster than Exp3-DOM and Exp3-Res which do not consider the uncertainty in the feedback graph. 
    
    
    \section{Conclusion}
    The present paper studied the problem of online learning with \emph{uncertain} feedback graphs, where potential uncertainties in the feedback graphs were modeled using probabilistic models.  Novel algorithms were developed to exploit information revealed by the nominal feedback graph 
    and different scenarios were discussed. 
    Specifically, in the informative case, where the probabilities associated with edges are also revealed, Exp3-IP was developed. 
    It is  proved that Exp3-IP can achieve sublinear regret bound. Furthermore, Exp3-UP and Exp3-GR were developed for the uninformative case.
  It is proved that Exp3-GR can achieve tighter sublinear regret bound than that of Exp3-UP when the number of experts is negligible compared to time horizon, while EXP3-UP requires less prior information than Exp3-GR. 
 Experiments on a number of real datasets were carried out to demonstrate that our novel algorithms can effectively address uncertainties in the feedback graph, and help enhance the learning ability of the learner.

	\bibliographystyle{plainnat}
	\bibliography{main}

\begin{thebibliography}{28}
\providecommand{\natexlab}[1]{#1}
\providecommand{\url}[1]{\texttt{#1}}
\expandafter\ifx\csname urlstyle\endcsname\relax
  \providecommand{\doi}[1]{doi: #1}\else
  \providecommand{\doi}{doi: \begingroup \urlstyle{rm}\Url}\fi

\bibitem[Alon et~al.(2015)Alon, Cesa-Bianchi, Dekel, and Koren]{Alon2015}
Noga Alon, Nicolò Cesa-Bianchi, Ofer Dekel, and Tomer Koren.
\newblock Online learning with feedback graphs: Beyond bandits.
\newblock In \emph{Proceedings of Conference on Learning Theory}, volume~40,
  pages 23--35, Paris, France, Jul 2015.

\bibitem[Alon et~al.(2017)Alon, Cesa-Bianchi, Gentile, Mannor, Mansour, and
  Shamir]{Alon2017}
Noga Alon, Nicolò Cesa-Bianchi, Claudio Gentile, Shie Mannor, Yishay Mansour,
  and Ohad Shamir.
\newblock Nonstochastic multi-armed bandits with graph-structured feedback.
\newblock \emph{SIAM Journal on Computing}, 46\penalty0 (6):\penalty0
  1785--1826, 2017.

\bibitem[Arora et~al.(2019)Arora, Marinov, and Mohri]{Arora2019}
Raman Arora, Teodor~Vanislavov Marinov, and Mehryar Mohri.
\newblock Bandits with feedback graphs and switching costs.
\newblock In \emph{Advances in Neural Information Processing Systems}, pages
  10397--10407, Dec 2019.

\bibitem[Auer et~al.(2003)Auer, Cesa-Bianchi, Freund, and Schapire]{Auer2003}
Peter Auer, Nicol\`{o} Cesa-Bianchi, Yoav Freund, and Robert~E. Schapire.
\newblock The nonstochastic multiarmed bandit problem.
\newblock \emph{SIAM Journal on Computing}, 32\penalty0 (1):\penalty0 48–77,
  Jan 2003.

\bibitem[Cesa-Bianchi et~al.(1997)Cesa-Bianchi, Freund, Haussler, Helmbold,
  Schapire, and Warmuth]{Cesa-Bianchi1997}
Nicol\`{o} Cesa-Bianchi, Yoav Freund, David Haussler, David~P. Helmbold,
  Robert~E. Schapire, and Manfred~K. Warmuth.
\newblock How to use expert advice.
\newblock \emph{Journal of the ACM}, 44\penalty0 (3):\penalty0 427–485, May
  1997.

\bibitem[Cesa-Bianchi et~al.(2020)Cesa-Bianchi, Cesari, and
  Monteleoni]{Cesa-Bianchi2020}
Nicol\`o Cesa-Bianchi, Tommaso Cesari, and Claire Monteleoni.
\newblock Cooperative online learning: Keeping your neighbors updated.
\newblock In \emph{Proceedings of the International Conference on Algorithmic
  Learning Theory}, volume 117, pages 234--250, Feb 2020.

\bibitem[Cesa-Bianchi and Lugosi(2006)]{Cesa-Bianchi2006}
Nicolò Cesa-Bianchi and Gabor Lugosi.
\newblock \emph{Prediction, Learning, and Games}.
\newblock Cambridge University Press, USA, 2006.

\bibitem[Chvatal(1979)]{Chvatal1979}
Vasek Chvatal.
\newblock A greedy heuristic for the set-covering problem.
\newblock \emph{Mathematics of Operations Research}, 4\penalty0 (3):\penalty0
  233--235, Aug 1979.

\bibitem[Cohen et~al.(2016)Cohen, Hazan, and Koren]{Cohen2016}
Alon Cohen, Tamir Hazan, and Tomer Koren.
\newblock Online learning with feedback graphs without the graphs.
\newblock In \emph{Proceedings of International Conference on Machine
  Learning}, page 811–819, Jun 2016.

\bibitem[Cortes et~al.(2019)Cortes, Desalvo, Gentile, Mohri, and
  Yang]{Cortes2019}
Corinna Cortes, Giulia Desalvo, Claudio Gentile, Mehryar Mohri, and Scott Yang.
\newblock Online learning with sleeping experts and feedback graphs.
\newblock In \emph{Proceedings of International Conference on Machine
  Learning}, pages 1370--1378, Jun 2019.

\bibitem[Cortes et~al.(2020)Cortes, DeSalvo, Gentile, Mohri, and
  Zhang]{Cortes2020}
Corrina Cortes, Giulia DeSalvo, Claudio Gentile, Mehryar Mohri, and Ningshan
  Zhang.
\newblock Online learning with dependent stochastic feedback graphs.
\newblock In \emph{Proceedings of International Conference on Machine
  Learning}, Jul 2020.

\bibitem[Dua and Graff(2017)]{Dua2019}
Dheeru Dua and Casey Graff.
\newblock {UCI} machine learning repository, 2017.

\bibitem[Hazan and Megiddo(2007)]{Hazan2007}
Elad Hazan and Nimrod Megiddo.
\newblock Online learning with prior knowledge.
\newblock In \emph{Proceedings of Annual Conference on Learning Theory}, page
  499–513, Jun 2007.

\bibitem[Helmbold et~al.(2000)Helmbold, Littlestone, and Long]{Helmbold2000}
David~P. Helmbold, Nicholas Littlestone, and Philip~M. Long.
\newblock Apple tasting.
\newblock \emph{Information and Computation}, 161\penalty0 (2):\penalty0
  85–139, Sep 2000.

\bibitem[Kawala et~al.(2013)Kawala, Douzal-Chouakria, Gaussier, and
  Dimert]{kawala2013}
Fran{\c c}ois Kawala, Ahlame Douzal-Chouakria, Eric Gaussier, and Eustache
  Dimert.
\newblock {Pr{\'e}dictions d'activit{\'e} dans les r{\'e}seaux sociaux en
  ligne}.
\newblock In \emph{{4i{\`e}me conf{\'e}rence sur les mod{\`e}les et l'analyse
  des r{\'e}seaux : Approches math{\'e}matiques et informatiques}}, page~16,
  France, October 2013.

\bibitem[Koc\'{a}k et~al.(2014)Koc\'{a}k, Neu, Valko, and Munos]{kocak2014}
Tom\'{a}\v{s} Koc\'{a}k, Gergely Neu, Michal Valko, and R\'{e}mi Munos.
\newblock Efficient learning by implicit exploration in bandit problems with
  side observations.
\newblock In \emph{Proceedings of International Conference on Neural
  Information Processing Systems}, page 613–621, Dec 2014.

\bibitem[Koc\'{a}k et~al.(2016)Koc\'{a}k, Neu, and Valko]{kocak2016}
Tom\'{a}\v{s} Koc\'{a}k, Gergely Neu, and Michal Valko.
\newblock Online learning with {E}rd\"{o}s-{R}\'{e}nyi side-observation graphs.
\newblock In \emph{Proceedings of Conference on Uncertainty in Artificial
  Intelligence}, page 339–346, Jun 2016.

\bibitem[Kocák et~al.(2016)Kocák, Neu, and Valko]{kocak2016noisy}
Tomáš Kocák, Gergely Neu, and Michal Valko.
\newblock Online learning with noisy side observations.
\newblock In \emph{Proceedings of International Conference on Artificial
  Intelligence and Statistics}, pages 1186--1194, Cadiz, Spain, May 2016.

\bibitem[Littlestone and Warmuth(1994)]{LITTLESTONE1994}
Nick Littlestone and Manfred~K. Warmuth.
\newblock The weighted majority algorithm.
\newblock \emph{Information and Computation}, 108\penalty0 (2):\penalty0 212 --
  261, 1994.

\bibitem[Liu et~al.(2018)Liu, Buccapatnam, and Shroff]{Liu2018}
Fang Liu, Swapna Buccapatnam, and Ness~B. Shroff.
\newblock Information directed sampling for stochastic bandits with graph
  feedback.
\newblock In \emph{Proceedings of AAAI Conference on Artificial Intelligence},
  Feb 2018.

\bibitem[Mannor and Shamir(2011)]{Mannor2011}
Shie Mannor and Ohad Shamir.
\newblock From bandits to experts: On the value of side-observations.
\newblock In \emph{Proc. of International Conference on Neural Information
  Processing Systems}, pages 684--692, 2011.

\bibitem[McQuade and Monteleoni(2012)]{McQuade2012}
Scott McQuade and Claire Monteleoni.
\newblock Global climate model tracking using geospatial neighborhoods.
\newblock In \emph{Proceedings of the AAAI Conference on Artificial
  Intelligence}, volume~26, pages 335--341, Jul 2012.

\bibitem[Papoulis and Pillai(2002)]{Popoulis2002}
Athanasios Papoulis and S.~Unnikrishna Pillai.
\newblock \emph{Probability, random variables, and stochastic processes}.
\newblock McGraw-Hill, 4th edition, 2002.

\bibitem[Rangi and Franceschetti(2019)]{Rangi2019}
Anshuka Rangi and Massimo Franceschetti.
\newblock Online learning with feedback graphs and switching costs.
\newblock In \emph{Proceedings of International Conference on Artificial
  Intelligence and Statistics}, pages 2435--2444, Apr 2019.

\bibitem[Resler and Mansour(2019)]{Resler2019}
Alon Resler and Yishay Mansour.
\newblock Adversarial online learning with noise.
\newblock In \emph{Proceedings of International Conference on Machine
  Learning}, pages 5429--5437, Jun 2019.

\bibitem[Tüfekci(2014)]{Tufekci2014}
Pınar Tüfekci.
\newblock Prediction of full load electrical power output of a base load
  operated combined cycle power plant using machine learning methods.
\newblock \emph{International Journal of Electrical Power and Energy Systems},
  60:\penalty0 126 -- 140, 2014.

\bibitem[Vito et~al.(2008)Vito, Massera, Piga, Martinotto, and
  Francia]{DeVito2008}
Saverio~De Vito, Ettore Massera, Marco Piga, Luca Martinotto, and Girolamo~Di
  Francia.
\newblock On field calibration of an electronic nose for benzene estimation in
  an urban pollution monitoring scenario.
\newblock \emph{Sensors and Actuators B: Chemical}, 129\penalty0 (2):\penalty0
  750 -- 757, 2008.

\bibitem[Yurinskiĭ(1976)]{Yurinski1976}
V.V Yurinskiĭ.
\newblock Exponential inequalities for sums of random vectors.
\newblock \emph{Journal of Multivariate Analysis}, 6\penalty0 (4):\penalty0 473
  -- 499, Dec 1976.

\end{thebibliography}

\newpage
\onecolumn
\appendix

	\section{Proof of Theorem \ref{th:1}} \label{A:1}
    Recall that $W_t = \sum_{i=1}^{K}{w_{i,t}}$ (below \eqref{eq:5}), we have
	\begin{align}
	\frac{W_{t+1}}{W_t} = \sum_{i=1}^{K}{ \frac{w_{i,t+1}}{W_t} } = \sum_{i=1}^{K}{ \frac{w_{i,t}}{W_t}\exp\left(-\eta \hat{\ell}_{t}(v_i)\right) }. \label{eq:1ap}
	\end{align}
	According to \eqref{eq:5}, we can write
	\begin{align}
	\frac{w_{i,t}}{W_t} = \frac{\pi_{i,t}-\eta \bar{F}_{i,t}}{1-\eta} \label{eq:2ap}
	\end{align}
	where $\bar{F}_{i,t} = \frac{F_{i,t}}{\sum_{j \in \mathcal{D}_t}{F_{j,t}}}\mathcal{I}(v_i \in \mathcal{D}_t)$. Substituting \eqref{eq:2ap} into \eqref{eq:1ap} obtains
	\begin{align}
	\frac{W_{t+1}}{W_t} = \sum_{i=1}^{K}{ \frac{\pi_{i,t}-\eta \bar{F}_{i,t}}{1-\eta}\exp\left(-\eta \hat{\ell}_{t}(v_i)\right) }. \label{eq:3ap}
	\end{align} 
	Using the inequality $e^{-x} \le 1-x+\frac{1}{2}x^{2}, \forall x \ge 0$,  the following inequality holds
	\begin{align}
	\frac{W_{t+1}}{W_t} \le  \sum_{i=1}^{K}{ \frac{\pi_{i,t}-\eta \bar{F}_{i,t}}{1-\eta}\left(1 -\eta \hat{\ell}_{t}(v_i) + \frac{1}{2}(\eta \hat{\ell}_{t}(v_i))^2\right) }. \label{eq:4ap}
	\end{align}
	Taking  logarithm of both sides of  \eqref{eq:4ap} and using the fact that $1+x \le e^x$, we have
	\begin{align}
	\ln \frac{W_{t+1}}{W_t} \le \sum_{i=1}^{K}{ \frac{\pi_{i,t}-\eta \bar{F}_{i,t}}{1-\eta}\left( -\eta \hat{\ell}_{t}(v_i) + \frac{1}{2}(\eta \hat{\ell}_{t}(v_i))^2\right) }. \label{eq:5ap}
	\end{align}
	Summing \eqref{eq:5ap} over time obtains
	\begin{align}
	\ln \frac{W_{T+1}}{W_1} \le \sum_{t=1}^{T}{ \sum_{i=1}^{K}{ \frac{\pi_{i,t}-\eta \bar{F}_{i,t}}{1-\eta}\left( -\eta \hat{\ell}_{t}(v_i) + \frac{1}{2}(\eta \hat{\ell}_{t}(v_i))^2\right) } }. \label{eq:6ap}
	\end{align}
	Furthermore, the left hand side of \eqref{eq:5ap} can be bounded from below as
	\begin{align}
	\ln \frac{W_{T+1}}{W_1} \ge \ln \frac{w_{i,T+1}}{W_1} = - \eta \sum_{t=1}^{T}{ \hat{\ell}_{t}(v_i)} - \ln K \label{eq:7ap}
	\end{align}
	where the equality holds due to the fact that $W_1 = \sum_{j=1}^{K}{w_{j,1}} = K$.
Then, \eqref{eq:6ap} and \eqref{eq:7ap} lead to
	\begin{align}
	& \sum_{t=1}^{T}{ \sum_{i=1}^{K}{ \frac{\eta \pi_{i,t}}{(1-\eta)} \hat{\ell}_{t}(v_i) } } - \eta \sum_{t=1}^{T}{ \hat{\ell}_{t}(v_i)} \nonumber \\  \le  & \ln K + \sum_{t=1}^{T}{\sum_{i \in \mathcal{D}_t}{ \frac{\eta^2 \bar{F}_{i,t}}{(1-\eta)} \hat{\ell}_{t}(v_i) }} + \sum_{t=1}^{T}{\sum_{i=1}^{K}{ \eta^2 \frac{\pi_{i,t}-\eta \bar{F}_{i,t}}{2(1-\eta)} \hat{\ell}_{t}(v_i)^2 }}. \label{eq:8ap}
	\end{align}
	Multiplying both sides of \eqref{eq:8ap} by $\frac{(1- \eta)}{\eta}$
	\begin{align}
	& \sum_{t=1}^{T}{ \sum_{i=1}^{K}{ \pi_{i,t} \hat{\ell}_{t}(v_i) } } - \sum_{t=1}^{T}{ \hat{\ell}_{t}(v_i)} \nonumber \\ \le & \frac{\ln K}{\eta} + \sum_{t=1}^{T}{\sum_{i \in \mathcal{D}_t}{ \eta \bar{F}_{i,t} \hat{\ell}_{t}(v_i) }} + \sum_{t=1}^{T}{\sum_{i=1}^{K}{ \frac{\eta}{2} (\pi_{i,t}-\eta \bar{F}_{i,t}) \hat{\ell}_{t}(v_i)^2 }}. \label{eq:9ap}
	\end{align}
 Furthermore, the expected values of $\hat{\ell}_{t}(v_i)$ and $\hat{\ell}_{t}(v_i)^2$  can be written as
	\begin{subequations} \label{eq:10ap}
		\begin{align} 
		\mathbb{E}_{t}[{\hat{\ell}_{t}(v_i)}] & = \sum_{j=1}^{K}{ \pi_{j,t}p_{ji,t} \frac{\ell_{t}(v_i)}{q_{i,t}} } = \ell_{t}(v_i) \label{eq:10apa} \\
		\mathbb{E}_{t}[{\hat{\ell}_{t}(v_i)}^2] & = \sum_{j=1}^{K}{ \pi_{j,t}p_{ji,t} \frac{\ell_{t}(v_i)^2}{q_{i,t}^2} } = \frac{\ell_{t}(v_i)^2}{q_{i,t}} \le \frac{1}{q_{i,t}} \label{eq:10apb}
		\end{align}
	\end{subequations}
	where the inequality in \eqref{eq:10apb} holds because of (a1) which implies $\ell_{t}(v_i) \le 1$. Taking the expectation of both sides of \eqref{eq:9ap}, we arrive at
	\begin{align}
	& \sum_{t=1}^{T}{ \sum_{i=1}^{K}{ \pi_{i,t} \ell_{t}(v_i) } } - \sum_{t=1}^{T}{ \ell_{t}(v_i)} \nonumber \\ \le & \frac{\ln K}{\eta} + \sum_{t=1}^{T}{\sum_{i=1}^{K}{ \eta \bar{F}_{i,t} \ell_{t}(v_i) }} + \sum_{t=1}^{T}{\sum_{i=1}^{K}{ \frac{\eta}{2} (\pi_{i,t}-\eta \bar{F}_{i,t}) \frac{1}{q_{i,t}} }}. \label{eq:11ap}
	\end{align}
	Moreover, using the fact that $q_{i,t} \le 1$ we have
	\begin{align}
	\frac{\eta^2}{2}\sum_{t=1}^{T}{\sum_{i=1}^{K}{\frac{\bar{F}_{i,t}}{q_{i,t}}}} \ge \frac{\eta^2}{2}\sum_{t=1}^{T}{\sum_{i=1}^{K}{ \bar{F}_{i,t} }} = \frac{\eta^2}{2}\sum_{t=1}^{T}{ 1 } = \frac{\eta^2 T}{2}. \label{eq:12ap}
	\end{align}
	Furthermore, since based on (a1) $\ell_{t}(v_i) \le 1$, the second term on the RHS of \eqref{eq:11ap} can be bounded by
	\begin{align}
	\eta \sum_{t=1}^{T}{ \sum_{i=1}^{K}{ \bar{F}_{i,t} \ell_{t}(v_i) } } \le \eta \sum_{t=1}^{T}{ \sum_{i=1}^{K}{ \bar{F}_{i,t} } } = \eta \sum_{t=1}^{T}{1} = \eta T. \label{eq:13ap}
	\end{align}
	Combining \eqref{eq:12ap}, \eqref{eq:13ap} with \eqref{eq:11ap} we have
	\begin{align}
	\sum_{t=1}^{T}{ \sum_{i=1}^{K}{ \pi_{i,t} \ell_{t}(v_i) } } - \sum_{t=1}^{T}{ \ell_{t}(v_i)} \le \frac{\ln K}{\eta} + \eta T - \frac{\eta^2 T}{2} + \frac{\eta}{2} \sum_{t=1}^{T}{ \sum_{i=1}^{K}{\frac{\pi_{i,t}}{q_{i,t}}} }. \label{eq:14ap}
	\end{align}
	By definition, the first term on the RHS of \eqref{eq:14ap} equals to $\mathbb{E}_t[\ell_t(v_{I_t})]$. In addition, note that  \eqref{eq:14ap} holds for all $v_i \in \mathcal{V}$, hence the following inequality holds
	\begin{align}
	\sum_{t=1}^{T}{\mathbb{E}_t[\ell_t(v_{I_t})]}-\min_{v_i \in \mathcal{V}}{ \sum_{t=1}^{T}{\ell_t(v_i)} } \le \frac{\ln K}{\eta} + \eta (1 - \frac{\eta}{2})T + \frac{\eta}{2} \sum_{t=1}^{T}{ \sum_{i=1}^{K}{\frac{\pi_{i,t}}{q_{i,t}}} } \label{eq:15ap}
	\end{align}
	which completes the proof of Theorem \ref{th:1}.
	
	\section{Proof of Lemma \ref{lem:4}} \label{B}
	Based on Theorem \ref{th:1}, the upper bound of the expected regret of Exp3-IP is
	\begin{align}
	\sum_{t=1}^{T}{\mathbb{E}_t[\ell_t(v_{I_t})]}-\min_{v_i \in \mathcal{V}}{ \sum_{t=1}^{T}{\ell_t(v_i)} } \le \frac{\ln K}{\eta} + \eta (1 - \frac{\eta}{2})T + \frac{\eta}{2} \sum_{t=1}^{T}{ \sum_{i=1}^{K}{\frac{\pi_{i,t}}{q_{i,t}}} }. \label{eq:16ap}
	\end{align}
	Let at each time instant $t$, $Q_{t}$ is defined as
	\begin{align}
	    Q_{t} = 1 + \frac{1}{2}\sum_{i=1}^{K}{\frac{\pi_{i,t}}{q_{i,t}}}. \label{eq:127ap}
	\end{align}
	Furthermore, let $\tau_r$ represent the greatest time instant such that $\sum_{t=1}^{\tau_r}{Q_t} \le 2^r$. According to the doubling trick at time instant $\tau_{r-1}+1$ where $\sum_{t=1}^{\tau_{r-1}+1}{Q_{t}} > 2^{r-1}$, the algorithm restarts with
	\begin{align}
	    \eta_r = \sqrt{\frac{\ln K}{2^r}}. \label{eq:128ap}
	\end{align}
	Also, the algorithm starts with $r=0$. Therefore, based on \eqref{eq:16ap} and \eqref{eq:128ap}, it can be concluded that
	\begin{align}
	    \sum_{t=1}^{\tau_r}{\pi_{i,t}\ell_t(v_i)} - \min_{v_i \in \mathcal{V}}{ \sum_{t=1}^{\tau_r}{\ell_t(v_i)} } \le 2 \sqrt{2^r \ln K} - \frac{\ln K}{2^{r+1}} \tau_r \label{eq:129ap}
	\end{align}
	when $2^{r-1} < \sum_{t=1}^{\tau_r}{Q_t} \le 2^r$. The maximum number of restarts required in this case is $\left \lceil{\log_2 \sum_{t=1}^{T}{Q_{t}}} \right \rceil$. Moreover, it can be written that
	\begin{align}
	    \sum_{r=0}^{\left \lceil{\log_2 \sum_{t=1}^{T}{Q_{t}}} \right \rceil}{2 \sqrt{2^r \ln K}} < \frac{4\sqrt{ \ln K}}{\sqrt{2}-1}\sqrt{\sum_{t=1}^{T}{Q_{t}}}. \label{eq:130ap}
	\end{align}
	Therefore, based on \eqref{eq:16ap} and considering the fact that the maximum possible value for incurred loss at each restart is $1$, combining \eqref{eq:129ap} with \eqref{eq:130ap} leads to
	\begin{align}
	    & \sum_{t=1}^{T}{\mathbb{E}_t[\ell_t(v_{I_t})]}-\min_{v_i \in \mathcal{V}}{ \sum_{t=1}^{T}{\ell_t(v_i)} } \nonumber \\ \le & {\mathcal{O}\left(\sqrt{(\ln K)\sum_{t=1}^{T}{Q_{t}}} + \left \lceil{\log_2 \sum_{t=1}^{T}{Q_{t}}} \right \rceil\right)} \nonumber \\ =&  {\mathcal{O}\left(\sqrt{\ln K\sum_{t=1}^{T}{(1+\frac{1}{2}\sum_{i=1}^{K}{\frac{\pi_{i,t}}{q_{i,t}}})}} + \left \lceil{\log_2 \sum_{t=1}^{T}{Q_{t}}} \right \rceil\right)} \label{eq:131ap}
	\end{align}
	Based on (a2), we can write $p_{ij} \ge \epsilon > 0$ if $(i,j) \in \mathcal{E}_t$. According to \eqref{eq:6} and the fact that the $i$-th expert is chosen by the learner with probability of $\pi_{i,t}$, based on (a2) the inequality $q_{i,t} \ge \pi_{i,t}\epsilon$ holds. Thus, we have
	\begin{align}
	{\left \lceil{\log_2 \sum_{t=1}^{T}{Q_{t}}} \right \rceil} = \mathcal{O}\left(\ln (\frac{K}{\epsilon}T)\right). \label{eq:17ap}
	\end{align}
    Combining \eqref{eq:131ap} with \eqref{eq:17ap} obtains
	\begin{align}
	    & \sum_{t=1}^{T}{\mathbb{E}_t[\ell_t(v_{I_t})]}-\min_{v_i \in \mathcal{V}}{ \sum_{t=1}^{T}{\ell_t(v_i)} } \nonumber \\ \le & \mathcal{O}\left(\sqrt{\ln K\sum_{t=1}^{T}{(1+\frac{1}{2}\sum_{i=1}^{K}{\frac{\pi_{i,t}}{q_{i,t}}})}} + \ln (\frac{K}{\epsilon}T)\right) \label{eq:133ap}
	\end{align}
	In order to move forward, the following Lemma is exploited \citep{Alon2017}.
	\begin{lemma} \label{lem:3}
	Let $\mathcal{G} = (\mathcal{V},\mathcal{E})$ be a directed graph with a set of vertices $\mathcal{V}$ and a set of edges $\mathcal{E}$. Let $\mathcal{D} \subseteq \mathcal{V}$ be a dominating set for $\mathcal{G}$ and $p_1,\ldots,p_K$ be a probability distribution defined over $\mathcal{V}$, such that $p_i \ge \beta > 0$, for $i \in \mathcal{D}$. Then
	\begin{align}
	    \sum_{i=1}^{K}{\frac{p_i}{\sum_{j:j \rightarrow i}{p_j}}} \le 2 \alpha(\mathcal{G}) \ln(1+\frac{\left \lceil{\frac{K^2}{\beta |\mathcal{D}|}} \right \rceil +K}{\alpha(\mathcal{G})}) + 2|\mathcal{D}| \label{eq:135ap}
	\end{align}
	where $\alpha(\mathcal{G})$ represents independence number for the graph $\mathcal{G}$.
	\end{lemma}
	Based on Lemma \ref{lem:3} and (a2), we get
	\begin{align}
	    \sum_{i=1}^{K}{\frac{\pi_{i,t}}{\sum_{\forall j: j \in \mathcal{N}_{i,t}^{\text{in}}}{\pi_{j,t}}}} < 2 \alpha(\mathcal{G}_t) \ln(1+\frac{\left \lceil{\frac{K^3}{\eta \epsilon}} \right \rceil +K}{\alpha(\mathcal{G}_t)}) + 2|\mathcal{D}_t|. \label{eq:136ap}
	\end{align}
	Considering the fact that $q_{i,t} \ge \epsilon \sum_{\forall j: j \in \mathcal{N}_{i,t}^{\text{in}}}{\pi_{j,t}}$ which is induced by (a2), from \eqref{eq:136ap}, it can be inferred that
	\begin{align}
	    \sum_{i=1}^{K}{\frac{\pi_{i,t}}{q_{i,t}}} < \frac{2 \alpha(\mathcal{G}_t)}{\epsilon} \ln(1+\frac{\left \lceil{\frac{K^3}{\eta \epsilon}} \right \rceil +K}{\alpha(\mathcal{G}_t)}) + \frac{2|\mathcal{D}_t|}{\epsilon}. \label{eq:137ap}
	\end{align}
	Furthermore, if greedy set cover algorithm by \citet{Chvatal1979} is employed to obtain the dominating set $|\mathcal{D}_t|$, it can be written that \citep{Alon2017}
	\begin{align}
	    |\mathcal{D}_t|=\mathcal{O}(\alpha(\mathcal{G}_t)\ln K). \label{eq:134ap}
	\end{align}
	Therefore, from \eqref{eq:137ap} we can conclude that
	\begin{align}
	    \sum_{i=1}^{K}{\frac{\pi_{i,t}}{q_{i,t}}} \le \mathcal{O}\left(\frac{\alpha(\mathcal{G}_t)}{\epsilon}\ln(\frac{KT}{\epsilon})\right) \label{eq:138ap}
	\end{align}
	Combining \eqref{eq:133ap} with \eqref{eq:134ap} and \eqref{eq:138ap}, we arrive at
	\begin{align}
	    & \sum_{t=1}^{T}{\mathbb{E}_t[\ell_t(v_{I_t})]}-\min_{v_i \in \mathcal{V}}{ \sum_{t=1}^{T}{\ell_t(v_i)} } \nonumber \\ \le & \mathcal{O}\left( \sqrt{\ln K\ln(\frac{K}{\epsilon}T)\sum_{t=1}^{T}{\frac{\alpha(\mathcal{G}_t)}{\epsilon}}} + \ln (\frac{K}{\epsilon}T)\right) \label{eq:139ap}
	\end{align}
	which completes the proof of Lemma \ref{lem:4}.

	\section{Proof of Theorem \ref{th:2}} \label{E}   
	
	In order to prove Theorem \ref{th:2}, let's first consider when $t \le KM$, during which the learner chooses among experts in a deterministic fashion. The (expected) loss incurred can henceforth  be written as
	\begin{align}
	    \mathbb{E}_{t}[\ell_t(v_i)] = \ell_t(v_k). \label{eq:71ap}
	\end{align}
	Since $\ell_t(v_i) \le 1$, we have
	\begin{align}
	    \sum_{t=1}^{KM}{ \mathbb{E}_{t}[\ell_t(v_i)] } - \sum_{t=1}^{KM}{ \ell_t(v_i) } \le (K-1)M. \label{eq:37ap}
	\end{align}
	On the other hand, for any $t$ we have
	\begin{align}
		\frac{W_{t+1}}{W_t} = \sum_{i=1}^{K}{ \frac{w_{i,t+1}}{W_t} } = \sum_{i=1}^{K}{ \frac{w_{i,t}}{W_t}\exp\left(-\eta \tilde{\ell}_{t}(v_i)\right) }. \label{eq:54ap}
	\end{align}
	Recall \eqref{eq:10}, we have
	\begin{align}
		\frac{w_{i,t}}{W_t} = \frac{\pi_{i,t}-\eta \hat{\bar{F}}_{i,t}}{1-\eta} \label{eq:55ap}
	\end{align}
	where $\hat{\bar{F}}_{i,t} = \frac{\eta}{|\mathcal{D}|}\mathcal{I}(v_i \in \mathcal{D}_t)$. Combining \eqref{eq:54ap} with \eqref{eq:55ap} leads to
	\begin{align}
		\frac{W_{t+1}}{W_t} = \sum_{i=1}^{K}{ \frac{\pi_{i,t}-\eta \hat{\bar{F}}_{i,t}}{1-\eta}\exp\left(-\eta \tilde{\ell}_{t}(v_i)\right) }. \label{eq:56ap}
	\end{align} 
	Due to the fact $e^{-x} \le 1-x+\frac{1}{2}x^{2}, \forall x \ge 0$, the following inequality holds
	\begin{align}
		\frac{W_{t+1}}{W_t} \le  \sum_{i=1}^{K}{  \frac{\pi_{i,t}-\eta \hat{\bar{F}}_{i,t}}{1-\eta}\left(1 -\eta \tilde{\ell}_{t}(v_i) + \frac{1}{2}(\eta \tilde{\ell}_{t}(v_i))^2\right) }. \label{eq:57ap}
	\end{align}
	Taking logarithm and using the fact that $1+x \le e^x$, we obtain
	\begin{align}
		\ln \frac{W_{t+1}}{W_t} \le \sum_{i=1}^{K}{  \frac{\pi_{i,t}-\eta \hat{\bar{F}}_{i,t}}{1-\eta}\left( -\eta \tilde{\ell}_{t}(v_i) + \frac{1}{2}(\eta \tilde{\ell}_{t}(v_i))^2\right) }. \label{eq:58ap}
	\end{align}
    Telescoping \eqref{eq:58ap} from $t^\prime := KM+1$ to $T$ achieves
	\begin{align}
		\ln \frac{W_{T+1}}{W_{t^\prime}} \le \sum_{t=t^\prime}^{T}{ \sum_{i=1}^{K}{ \frac{\pi_{i,t}-\eta \hat{\bar{F}}_{i,t}}{1-\eta}\left( -\eta \tilde{\ell}_{t}(v_i) + \frac{1}{2}(\eta \tilde{\ell}_{t}(v_i))^2\right) } }. \label{eq:59ap}
	\end{align}
	Moreover, note that $\ln \frac{W_{T+1}}{W_{t^\prime}}$ can be bounded by
	\begin{align}
		\ln \frac{W_{T+1}}{W_{t^\prime}} \ge \ln \frac{w_{i,T+1}}{W_1} = -\eta \sum_{\forall t: t \notin \mathcal{M}}{ \tilde{\ell}_{t}(v_i)} - \ln K. \label{eq:60ap}
	\end{align}
	Combining \eqref{eq:59ap} with \eqref{eq:60ap} obtains
	\begin{align}
		& \sum_{t=t^\prime}^{T}{ \sum_{i=1}^{K}{ \frac{\eta \pi_{i,t}}{1-\eta} \tilde{\ell}_{t}(v_i) } } -\eta \sum_{t=t^\prime}^{T}{ \tilde{\ell}_{t}(v_i)} \nonumber \\  \le & \ln K + \sum_{t=t^\prime}^{T}{\sum_{i=1}^{K}{ \frac{\eta^2 \hat{\bar{F}}_{i,t}}{1-\eta} \tilde{\ell}_{t}(v_i) }} + \sum_{t=t^\prime}^{T}{\sum_{i=1}^{K}{ \eta^2 \frac{\pi_{i,t}-\eta \hat{\bar{F}}_{i,t}}{2(1-\eta)} \tilde{\ell}_{t}(v_i)^2 }}. \label{eq:61ap}
	\end{align}
	Multiplying both sides of \eqref{eq:61ap} by $\frac{1- \eta }{\eta}$ arrives at
	\begin{align}
	& \sum_{t=t^\prime}^{T}{ \sum_{i=1}^{K}{ \pi_{i,t} \tilde{\ell}_{t}(v_i) } } - \sum_{t=t^\prime}^{T}{ \tilde{\ell}_{t}(v_i)} \nonumber \\  \le & \frac{\ln K}{\eta} + \sum_{t=t^\prime}^{T}{\sum_{i=1}^{K}{ \eta \hat{\bar{F}}_{i,t} \tilde{\ell}_{t}(v_i) }} + \sum_{t=t^\prime}^{T}{\sum_{i=1}^{K}{ \frac{\eta}{2} (\pi_{i,t}-\eta \hat{\bar{F}}_{i,t}) \tilde{\ell}_{t}(v_i)^2 }}. \label{eq:62ap}
	\end{align} 
	
	In addition, the expected value of $\tilde{\ell}_{t}(v_i)$ and $\tilde{\ell}_{t}(v_i)^2$ at time instant $t$ can be written as
	\begin{subequations} \label{eq:63ap}
		\begin{align} 
		\mathbb{E}_{t}[{\tilde{\ell}_{t}(v_i)}] & = \sum_{\forall j: v_j \in \mathcal{N}_{i,t}^{\text{in}}}{ \pi_{j,t}p_{ji} \frac{1}{\hat{q}_{i,t}}\ell_{t}(v_i) } = \frac{q_{i,t}}{\hat{q}_{i,t}}\ell_{t}(v_i) \label{eq:63apa} \\
		\mathbb{E}_{t}[{\tilde{\ell}_{t}(v_i)}^2] & = \sum_{\forall j: v_j \in \mathcal{N}_{i,t}^{\text{in}}}{ \pi_{j,t}p_{ji} \frac{1}{\hat{q}_{i,t}^2}\ell_{t}(v_i)^2 } = \frac{q_{i,t}}{\hat{q}_{i,t}^2}\ell_{t}(v_i)^2 \le \frac{q_{i,t}}{\hat{q}_{i,t}^2}. \label{eq:63apb}
		\end{align}
	\end{subequations}

Let $e_{ij,t}:=|\hat{p}_{ij,t}-p_{ij}|$.  
	According to \eqref{eq:11}, the probability that $\hat{q}_{i,t} \ge q_{i,t}$ is at least $\prod_{\forall j: v_j \in \mathcal{N}_{i,t}^{\text{in}}}\text{Pr}(e_{ij,t}\le \xi/\sqrt{M})$ since the incidents $\{e_{ij,t}\le \xi/\sqrt{M}$, $\forall (i,j) \in \mathcal{E}\}$ are independent from each other. Let $\varepsilon$ denote $\xi/\sqrt{M}$
	and $\mu_{i,t}:=\frac{1}{\hat{q}_{i,t}} - \frac{1}{q_{i,t}} $, we have
	\begin{align}
	\mu_{i,t} =  \frac{q_{i,t}-\hat{q}_{i,t}}{\hat{q}_{i,t}q_{i,t}} =\frac{\sum_{\forall j: v_j \in \mathcal{N}_{i,t}^{\text{in}}}{\pi_{j,t}(p_{ji} - \hat{p}_{ji,t} - \varepsilon)}}{\hat{q}_{i,t}q_{i,t}} \ge  -\frac{\sum_{\forall j: v_j \in \mathcal{N}_{i,t}^{\text{in}}}{2\pi_{j,t}\varepsilon}}{q_{i,t}^2} \label{eq:38ap}
	\end{align} 
where the last inequality holds	with probability  $\prod_{\forall j: v_j \in \mathcal{N}_{i,t}^{\text{in}}}\text{Pr}(e_{ij,t}\le \varepsilon)$. Therefore, the following inequalities hold with the probability  $\prod_{\forall j: v_j \in \mathcal{N}_{i,t}^{\text{in}}}\text{Pr}(e_{ij,t}\le \varepsilon)$
	\begin{subequations} \label{eq:70ap}
	    \begin{align}
	 & \ell_{t}(v_i) - \sum_{\forall j: v_j \in \mathcal{N}_{i,t}^{\text{in}}}{\frac{2\pi_{j,t}\varepsilon}{q_{i,t}} \ell_{t}(v_i)} \nonumber \\ &\leq     \mathbb{E}_{t}[{\tilde{\ell}_{t}(v_i)}] = \ell_{t}(v_i) + q_{i,t}\mu_{i,t}\ell_{t}(v_i) \le \ell_t(v_i)  \label{eq:70apa} \\
	        & \mathbb{E}_{t}[{\tilde{\ell}_{t}(v_i)}^2] \le \frac{1}{q_{i,t}}. \label{eq:70apc}
	    \end{align}
	\end{subequations}
	Taking expectation of both sides of \eqref{eq:62ap} and combining with \eqref{eq:70ap}, we obtain the following inequality 
	\begin{align}
	& \sum_{t = t^\prime}^{T}{ \sum_{i=1}^{K}{ \pi_{i,t} \ell_{t}(v_i) } } - \sum_{t = t^\prime}^{T}{ \sum_{i=1}^{K}{ \pi_{i,t}\sum_{\forall j: v_j \in \mathcal{N}_{i,t}^{\text{in}}}{\frac{2\pi_{j,t}\varepsilon}{q_{i,t}} \ell_{t}(v_i)} } } - \sum_{t = t^\prime}^{T}{ \ell_{t}(v_i)} \nonumber \\
	\le & \frac{\ln K}{\eta} + \sum_{t = t^\prime}^{T}{\sum_{i=1}^{K}{ \eta \hat{\bar{F}}_{i,t} \ell_{t}(v_i) }} + \sum_{t = t^\prime}{\sum_{i=1}^{K}{ \frac{\eta}{2} (\pi_{i,t}-\eta \hat{\bar{F}}_{i,t}) \frac{1}{q_{i,t}} }} \nonumber \\
	\le & \frac{\ln K}{\eta} + \sum_{t = t^\prime}^{T}{\sum_{i=1}^{K}{ \eta \hat{\bar{F}}_{i,t} }} + \sum_{t = t^\prime}^{T}{\sum_{i=1}^{K}{ \frac{\eta}{2} (\pi_{i,t}-\eta \hat{\bar{F}}_{i,t}) \frac{1}{q_{i,t}} }} \label{eq:64ap}
	\end{align} 
	which holds with probability at least $\prod_{(i,j)\in \mathcal{E}_{t}}{\Pr(e_{ij,t^\prime} \le \varepsilon, \ldots, e_{ij,T} \le \varepsilon)}$.
	 Applying the chain rule for one term in the product, we have
	\begin{align}
	    & \Pr(e_{ij,t^\prime} \le \varepsilon, \ldots, e_{ij,T} \le \varepsilon) \nonumber \\  =& \Pr(e_{ij,t^\prime} \le \varepsilon)\prod_{t=t^\prime+1}^{T}{\Pr(e_{ij,t} \le \varepsilon \mid e_{ij,t-1} \le \varepsilon,\ldots,e_{ij,t^\prime} \le \varepsilon)}  \nonumber \\ \ge & \prod_{t=t^\prime}^{T}{ \Pr(e_{ij,t} \le \varepsilon) }. \label{eq:72ap} 
	\end{align}
	In order to obtain the lower bound of the probability $\Pr(e_{ij,t} \le \varepsilon_{ij,t})$, the Bernstein inequality is employed. To this end consider the following lemma \citep{Yurinski1976}.
	\begin{lemma} \label{lem:1}
	Let $\zeta_1,\ldots,\zeta_n$ be independent random variables such that
	\begin{subequations} \label{eq:108ap}
	    \begin{align}
	        \mathbb{E}[\zeta_i] & = 0, \forall i: 1 \le i \le n \label{eq:108apa} \\
	        |\mathbb{E}[\zeta_i^m]| & \le \frac{m!}{2}b_i^2 H^{m-2}, m=2,3,\ldots, \forall i: 1 \le i \le n. \label{eq:108apb}
	    \end{align}
	 \end{subequations}
	 Then for $x\ge 0$, we have
	 \begin{align}
	     \Pr(|\zeta_1+\ldots+\zeta_n| \ge xB_n) \le 2 \exp(-\frac{\frac{x^2}{2}}{1+\frac{xH}{B_n}}) \label{eq:110ap}
	 \end{align}
	 where $B_n^2 = b_1^2+\ldots+b_n^2$.
	\end{lemma}
	Let $\theta_{ij}(t) := X_{ij}(t)-p_{ij}$, $\forall (i,j) \in \mathcal{E}_t$. Since $X_{ij}(t)$ follows Bernoulli distribution with the parameter $p_{ij}$, it can be readily obtained that $\mathbb{E}[\theta_{ij}(t)]=0$. Furthermore, for the moment generating function of $\theta_{ij}(t)$, we have
	\begin{align}
	    M_{\theta_{ij}(t)}(z) = (1-p_{ij})e^{-p_{ij}z} + p_{ij}e^{(1-p_{ij})z}. \label{eq:111ap}
	\end{align}
	Therefore, the expected value of $\theta_{ij}^{m}(t)$, $m=2,3,\ldots$ can be expressed as
	\begin{align}
	    \mathbb{E}[\theta_{ij}^{m}(t)] = \frac{d^m M_{\theta_{ij}(t)}(z)}{dz^m}\mid_{z=0} = (-p_{ij})^m(1-p_{ij})+(1-p_{ij})^mp_{ij}. \label{eq:109ap}
	\end{align}
	From \eqref{eq:109ap}, we can conclude that
	\begin{align}
	    |\mathbb{E}[\theta_{ij}^{m}(t)]| \le p_{ij}(1-p_{ij}) \le \frac{1}{4} \le \frac{m!}{8} = \frac{m!}{2}\left(\frac{1}{2}\right)^2\times 1^{m-2}, m=2,3,\ldots \label{eq:112ap}
	\end{align}
	Thus, letting $b_i=\frac{1}{2}$, $H=1$ in Lemma \ref{lem:1} and combining with \eqref{eq:112ap},  the following inequality can be obtained
	\begin{align}
	    \Pr\left(|\sum_{\tau \in \mathcal{T}_{ij,t}}{\theta_{ij}(\tau)}| \ge \frac{\xi C_{ij,t}}{\sqrt{M}} \right) & = \Pr\left(|\sum_{\tau \in \mathcal{T}_{ij,t}}{X_{ij}(\tau)-p_{ij}}| \ge \frac{\xi C_{ij,t}}{\sqrt{M}} \right) \nonumber \\ & \le 2\exp(-\frac{2\xi^2 \frac{C_{ij,t}}{M}}{1+\frac{4\xi}{\sqrt{M}}}) \nonumber \\  & = 2 \exp(-\frac{2\xi^2 C_{ij,t}}{M+4\xi \sqrt{M}}) \label{eq:113ap}
	\end{align}
which leads to 
	\begin{align}
	\text{Pr}(e_{ij,t}\ge \varepsilon) \le 2 \exp(-\frac{2\xi^2 C_{ij,t}}{M+4\xi \sqrt{M}}) \label{eq:36ap} 
	\end{align}
Therefore,  \eqref{eq:64ap} holds with probability at least 
	\begin{align}
	    \delta_\xi = \prod_{t=t^\prime}^{T}{ \prod_{(i,j)\in\mathcal{E}_{t}}{ \left(1-2 \exp(-\frac{2\xi^2 C_{ij,t}}{M+4\xi \sqrt{M}})\right) } }. \label{eq:74ap}
	\end{align}
	Since $\sum_{\forall j: v_j \in \mathcal{N}_{i,t}^{\text{in}}}{\pi_{j,t}} \le 1$ and $\varepsilon = \frac{\xi}{\sqrt{M}}$, the following inequality holds
	\begin{align}
	\sum_{t=t^\prime}^{T}{ \sum_{i=1}^{K}{ \pi_{i,t}\sum_{\forall j: v_j \in \mathcal{N}_{i,t}^{\text{in}}}{\frac{2\pi_{j,t}\varepsilon}{q_{i,t}} } } } = & \sum_{t=t^\prime}^{T}{ \sum_{i=1}^{K}{ \pi_{i,t}\sum_{\forall j: v_j \in \mathcal{N}_{i,t}^{\text{in}}}{\frac{2\pi_{j,t}\frac{\xi}{\sqrt{M}}}{q_{i,t}} } } } \nonumber \\ \le & \sum_{t=t^\prime}^{T}{ \sum_{i=1}^{K}{ \frac{2\pi_{i,t}\xi}{q_{i,t}\sqrt{M}} } }. \label{eq:66ap}
	\end{align}
Using \eqref{eq:66ap} and the fact that $\frac{1}{q_{i,t}} \ge 1$, \eqref{eq:64ap} can be rewritten as 
	\begin{align}
	& \sum_{t=t^\prime}^{T}{ \sum_{i=1}^{K}{ \pi_{i,t} \ell_{t}(v_i) } } - \sum_{t=t^\prime}^{T}{ \ell_{t}(v_i)} \nonumber \\ \le & \frac{\ln K}{\eta} + \sum_{t=t^\prime}^{T}{ \eta(1-\frac{\eta}{2}) } + \sum_{t=t^\prime}^{T}{ \sum_{i=1}^{K}{ \frac{\pi_{i,t}}{q_{i,t}}(\frac{2\xi}{\sqrt{M}}+\frac{\eta}{2}) } } \label{eq:67ap}
	\end{align}
Combining  \eqref{eq:67ap} with \eqref{eq:37ap} results in following inequality
	\begin{align}
	& \sum_{t=1}^{T}{ \mathbb{E}_{t}[\ell_{t}(v_{I_t})] } - \sum_{t=1}^{T}{ \ell_{t}(v_i)} \nonumber \\ \le & \frac{\ln K}{\eta} + (K-1)M + \eta(1-\frac{\eta}{2})(T-KM) + \sum_{t=t^\prime}^{T}{ \sum_{i=1}^{K}{ \frac{\pi_{i,t}}{q_{i,t}}(\frac{2\xi}{\sqrt{M}}+\frac{\eta}{2}) } } \label{eq:68ap}
	\end{align}
	which holds with probability at least $\delta_\xi$ 
	and the proof of  Theorem \ref{th:2} is completed.

	
	\section{Proof of Corollary \ref{cor:3}} \label{G}
	
The proof of Corollary \ref{cor:3} will be built upon the following Lemma.
	\begin{lemma} \label{lem:2}
	Let $\zeta_1,\ldots,\zeta_N$ ($N>1$) be a sequence of real positive numbers such that $\forall i: 1 \le i \le N$, $0<\zeta_i<1$ and $\forall n: 1 \le n \le N$, $\sum_{i=1}^{n}{ \zeta_i } < 1$. Then, it can be written that
	\begin{align}
	    \prod_{i=1}^{N}(1-\zeta_i) > 1-\sum_{i=1}^{N}{ \zeta_i } \label{eq:119ap}
	\end{align}
	\end{lemma}
	\begin{proof}
	We prove this Lemma using mathematical induction. Firstly, Consider  \eqref{eq:119ap} for $N=2$ 
	\begin{align}
	    (1-\zeta_1)(1-\zeta_2) = 1-\zeta_1-\zeta_2+\zeta_1 \zeta_2 > 1-\zeta_1-\zeta_2. \label{eq:120ap}
	\end{align}
	Assuming that \eqref{eq:119ap} holds for $N=n$. Then, based on \eqref{eq:120ap} we have for $N=n+1$
	\begin{align}
	    \prod_{i=1}^{n+1}(1-\zeta_i) = \left(\prod_{i=1}^{n}{(1-\zeta_i)}\right) \times (1-\zeta_{n+1}) >& (1-\sum_{i=1}^{n}{ \zeta_i })(1-\zeta_{n+1}) \nonumber \\ > & 1-\sum_{i=1}^{n+1}{ \zeta_i }. \label{eq:121ap}
	\end{align}
	Hence, \eqref{eq:119ap}  also holds  for $N=n+1$, and Lemma \ref{lem:2} is proved by induction.
	\end{proof}
	Assuming $M$ satisfies 
	\begin{align}
	    M \ge (\frac{4\xi \ln (KT)}{\xi^2-\ln (KT)})^2. \label{eq:118ap}
	\end{align}
Hence, \eqref{eq:118ap} can be re-written as
	\begin{align}
	   \frac{1}{K^2T^2} \geq \exp(-\frac{2\xi^2 \sqrt{M}}{\sqrt{M}+4\xi}) \geq  \exp(-\frac{2\xi^2 C_{ij,t}}{M+4\xi \sqrt{M}}) \label{eq:116ap}
	\end{align}
	where the second inequality holds since $C_{ij,t} \ge M$.
	Let $t^\prime=KM+1$. Note that the regret bound in \eqref{eq:19} holds with probability at least $\delta_\xi$ in \eqref{eq:74ap}.  According to Lemma \ref{lem:2}, we can obtain the following inequality
	\begin{align}
	    \delta_\xi =& \prod_{t=t^\prime}^{T}{ \prod_{(i,j)\in\mathcal{E}_{t}}{ \left(1-2 \exp(-\frac{2\xi^2 C_{ij,t}}{M+4\xi \sqrt{M}})\right) } } \nonumber \\ > & 1 - \sum_{(i,j)\in\mathcal{E}_{t}}{ \sum_{t=t^\prime}^{T}{ 2\exp(-\frac{2\xi^2 C_{ij,t}}{M+4\xi \sqrt{M}}) } }. \label{eq:114ap}
	\end{align}
	Combining \eqref{eq:116ap} with \eqref{eq:114ap} obtains
	\begin{align}
	     \delta_\xi \geq 1-\frac{2(T-KM)|\mathcal{E}|}{K^2T^2} \label{eq:114ap2}
	\end{align}
	where $|\mathcal{E}|$ denotes the cardinality of the $\mathcal{E}$. Since $\mathcal{G}$ does not change over time, $|\mathcal{E}|$ is a constant. According to \eqref{eq:114ap2}, it can be readily obtained that when \eqref{eq:118ap} holds, the regret bound in \eqref{eq:19} holds with probability at least of order $1-\mathcal{O}(\frac{1}{T})$. 
    Consider the case where the learner sets $\eta$, $M$ and $\xi$ as follows
    \begin{subequations} \label{eq:140ap}
    \begin{align}
        \eta & = \mathcal{O}(\sqrt{\frac{\ln K}{T}}) \label{eq:140apa} \\
        M & = \mathcal{O}(\frac{1}{\sqrt{K}}T^{\frac{2}{3}}) \label{eq:140apb} \\
        \xi & = \mathcal{O}(K^{\frac{1}{4}}\sqrt{\ln (KT)}). \label{eq:140apc}
    \end{align}
    \end{subequations}
    Putting $\eta$, $M$ and $\xi$ in \eqref{eq:140ap} into \eqref{eq:19} and based on Lemma \ref{lem:3}, it can be concluded that the expected regret of Exp3-UP satisfies
    \begin{align}
        & \sum_{t=1}^{T}{ \mathbb{E}_{t}[\ell_{t}(v_{I_t})] } - \sum_{t=1}^{T}{ \ell_{t}(v_i)} \nonumber \\ \le & \mathcal{O}\left(\frac{\alpha(\mathcal{G})}{\epsilon}\ln(KT)(\sqrt{T \ln K} + \sqrt{K\ln (KT)} T^{\frac{2}{3}})\right) \label{eq:141ap}
    \end{align}
    with probability at least $1-\mathcal{O}(\frac{1}{T})$. 
    
    \section{Proof of Lemma \ref{lem:5}} \label{H}
    In this section, doubling trick technique is employed such that Exp3-UP can achieve sub-linear regret. If $2^b < t \le 2^{b+1}$, the value of the learning rate $\eta_b$, $M_b$ and $\xi_b$ are
    \begin{subequations} \label{eq:142ap}
    \begin{align}
        \eta_b & = \sqrt{\frac{\ln K}{2^{b+1}}} \label{eq:142apa} \\
        M_b & = \left \lceil{2^{\frac{2(b+1)}{3}}\frac{1}{\sqrt{K}} + \ln 4K} \right \rceil \label{eq:142apb} \\
        \xi_b & = \left(2K^{\frac{1}{4}}+\sqrt{4\sqrt{K}+1}\right)\sqrt{\ln (K2^{b+3})} \label{eq:142apc}
    \end{align}
    \end{subequations}
    When the learner realizes that $t > 2^{b+1}$, the algorithm restarts with $\eta_{b+1}$, $M_{b+1}$ and $\xi_{b+1}$. The algorithm starts with $b=\left \lceil{\log_2 K} \right \rceil$. Therefore, when $t < 2^{\left \lceil{\log_2 K} \right \rceil}$, the value of $\eta_b$, $M_b$ and $\xi_b$ are set with respect to $b=\left \lceil{\log_2 K} \right \rceil$. Let $\mathcal{M}_i$ denotes a set which includes the time instants when the learner chooses the $i$-th expert in a deterministic fashion for exploration. Specifically, when at time instant $\tau$, the learner chooses the $i$-th expert for exploration without using the PMF in \eqref{eq:10}, the time instant $\tau$ is appended to $\mathcal{M}_i$. At each restart the learner chooses the experts one by one for the exploration until the condition $|\mathcal{M}_i| \ge M_b$, $\forall i \in [K]$ is satisfied. Then, the learner chooses between experts according to PMF in \eqref{eq:10} using the learning rate $\eta_b$. Therefore, based on Theorem \ref{th:2}, for each $b$, the algorithm satisfies
    \begin{align}
        & \sum_{t=2^b+1}^{T_b}{ \mathbb{E}_{t}[\ell_{t}(v_{I_t})] } - \sum_{t=2^b+1}^{T_b}{ \ell_{t}(v_i)} \nonumber \\ \le & \frac{\ln K}{\eta_b} + (K-1)(M_b-M_{b-1}) + \eta_b(1-\frac{\eta_b}{2})(T_b-2^{b}-K(M_b-M_{b-1})) \nonumber \\ & + \sum_{t=2^b +1}^{T_b}{ \sum_{i=1}^{K}{ \frac{\pi_{i,t}}{q_{i,t}}(\frac{2\xi_b}{\sqrt{M_b}}+\frac{\eta_b}{2}) } } \label{eq:143ap}
    \end{align}
    with probability at least $\delta_b$ where it can be expressed as
    \begin{align}
        \delta_b = \prod_{t=t_b^\prime}^{T_b}{ \prod_{(i,j)\in\mathcal{E}_{t}}{ \left(1-2 \exp(-\frac{2\xi_b^2 C_{ij,t}}{M_b+4\xi_b \sqrt{M_b}})\right) } } \label{eq:147ap}
    \end{align}
    where $T_b$ denote the greatest time instant which satisfies $2^b < T_b \le 2^{b+1}$ and $t_b^\prime$ can be written as
    \begin{align}
        t_b^\prime = \min(T_{b-1} + K(M_b-M_{b-1})+1,T_b). \label{eq:149ap}
    \end{align}
    Note that $M_{b-1}=0$ when $b=\left \lceil{\log_2 K} \right \rceil$. Since for each $b$, $M_b$ and $\xi_b$ in \eqref{eq:142ap} meet the following condition
    \begin{align}
        M_b \ge (\frac{4\xi_b \ln (4KT_b)}{\xi_b^2-\ln (4KT_b)})^2, \label{eq:148ap}
    \end{align}
    it can be concluded that the following inequality holds true
    \begin{align}
        \frac{1}{16K^2T_b^2} \geq \exp(-\frac{2\xi_b^2 \sqrt{M_b}}{\sqrt{M_b}+4\xi_b}) \geq  \exp(-\frac{2\xi_b^2 C_{ij,t}}{M_b+4\xi_b \sqrt{M_b}}), \label{eq:150ap}
    \end{align}
    and as a result according to Lemma \ref{lem:2} we can write
    \begin{align}
        \delta_b > 1 - \sum_{(i,j)\in\mathcal{E}_{t}}{ \sum_{t=t_b^\prime}^{T_b}{ 2\exp(-\frac{2\xi_b^2 C_{ij,t}}{M_b+4\xi_b \sqrt{M_b}}) } }. \label{eq:151ap}
    \end{align}
    Combining \eqref{eq:150ap} with \eqref{eq:151ap}, it can be concluded that
    \begin{align}
        \delta_b > 1- \max(0,\frac{(T_b-t_b^\prime)|\mathcal{E}_{T_b}|}{8K^2T_b^2}). \label{eq:152ap}
    \end{align}
    Therefore, for each $b$ from \eqref{eq:143ap}, \eqref{eq:152ap} and Lemma \ref{lem:3} it can be inferred that
    \begin{align}
        & \sum_{t=2^b+1}^{T_b}{ \mathbb{E}_{t}[\ell_{t}(v_{I_t})] } - \sum_{t=2^b+1}^{T_b}{ \ell_{t}(v_i)} \nonumber \\ \le & \mathcal{O}\left(\frac{\alpha(\mathcal{G})}{\epsilon}\ln(KT_b)(\sqrt{T_b \ln K} + \sqrt{K\ln(KT_b)} T_b^{\frac{2}{3}})\right) \label{eq:144ap}
    \end{align}
    holds with probability at least $1-\mathcal{O}(\frac{1}{T_b})$. Summing \eqref{eq:144ap} over all possible values of $b$, from $b:=\left \lceil{\log_2 K} \right \rceil$ to $\left \lceil{\log_2 T} \right \rceil$ and taking into account that the maximum value of the loss at each restart is $1$, we arrive at
    \begin{align}
        & \sum_{t=1}^{T}{ \mathbb{E}_{t}[\ell_{t}(v_{I_t})] } - \sum_{t=1}^{T}{ \ell_{t}(v_i)} \nonumber \\ \le & \!\!\sum_{b=\left \lceil{\log_2 K} \right \rceil}^{\left \lceil{\log_2 T} \right \rceil}{\mathcal{O}\left(\frac{\alpha(\mathcal{G})}{\epsilon}\ln(KT_b)(\sqrt{T_b \ln K} + \sqrt{K\ln(KT_b)} T_b^{\frac{2}{3}})\right)} + \left \lceil{\log_2 T} \right \rceil - \left \lceil{\log_2 K} \right \rceil \nonumber \\ \le & \mathcal{O}\left(\frac{\alpha(\mathcal{G})}{\epsilon}\ln(T)\ln(KT)(\sqrt{T \ln K} + \sqrt{K\ln(KT)} T^{\frac{2}{3}}) + \ln T \right) \label{eq:145ap}
    \end{align}
    which holds with probability at least
    \begin{align}
        \Delta = \prod_{b=\left \lceil{\log_2 K} \right \rceil}^{\left \lceil{\log_2 T} \right \rceil}{\left(1-\max(0,\frac{(T_b-t_b^\prime)|\mathcal{E}_{T_b}|}{8K^2T_b^2})\right)}. \label{eq:146ap}
    \end{align}
    When $b=\left \lceil{\log_2 K} \right \rceil$, we have $T_b \ge 2K$. Furthermore, when $\left \lceil{\log_2 K} \right \rceil < b \le \left \lfloor{\log_2 T} \right \rfloor$, it can be concluded that $T_b = 2T_{b-1}$. Therefore, we can write
    \begin{align}
        \sum_{b=\left \lceil{\log_2 K} \right \rceil}^{\left \lfloor{\log_2 T} \right \rfloor}{\max(0,\frac{(T_b-t_b^\prime)|\mathcal{E}_{T_b}|}{8K^2T_b^2})} & < \sum_{b=\left \lceil{\log_2 K} \right \rceil}^{\left \lfloor{\log_2 T} \right \rfloor}{\frac{1}{8T_b}} \nonumber \\ & \le \frac{1}{8K}(\sum_{b=\left \lceil{\log_2 K} \right \rceil}^{\left \lfloor{\log_2 T} \right \rfloor}{(\frac{1}{2})^{b-\left \lceil{\log_2 K} \right \rceil}}) \nonumber \\ & = \frac{1}{8K}(2-(\frac{1}{2})^{\left \lfloor{\log_2 T} \right \rfloor - \left \lceil{\log_2 K} \right \rceil}). \label{eq:153ap}
    \end{align}
    Hence, based on \eqref{eq:153ap} and under the assumption that $T>K$, we find
    \begin{align}
        \sum_{b=\left \lceil{\log_2 K} \right \rceil}^{\left \lceil{\log_2 T} \right \rceil}{\max(0,\frac{(T_b-t_b^\prime)|\mathcal{E}_{T_b}|}{8K^2T_b^2})} < & \frac{1}{8K}(2-(\frac{1}{2})^{\left \lfloor{\log_2 T} \right \rfloor - \left \lceil{\log_2 K} \right \rceil}) + \frac{1}{8T} \nonumber \\ < & \frac{3}{8K}. \label{eq:154ap}
    \end{align}
    Thus, $\Delta$ meet the conditions in the Lemma \ref{lem:2} and it can be inferred that
    \begin{align}
        \Delta > 1- \sum_{b=\left \lceil{\log_2 K} \right \rceil}^{\left \lceil{\log_2 T} \right \rceil}{\max(0,\frac{(T_b-t_b^\prime)|\mathcal{E}_{T_b}|}{8K^2T_b^2})} \ge 1-\mathcal{O}(\frac{1}{K}). \label{eq:155ap}
    \end{align}
    Therefore, in this case, Exp3-UP satisfies
    \begin{align}
        & \sum_{t=1}^{T}{ \mathbb{E}_{t}[\ell_{t}(v_{I_t})] } - \sum_{t=1}^{T}{ \ell_{t}(v_i)} \nonumber \\ \le & \mathcal{O}\left(\frac{\alpha(\mathcal{G})}{\epsilon}\ln(T)\ln(KT)(\sqrt{T \ln K} + \sqrt{K\ln(KT)} T^{\frac{2}{3}}) + \ln T \right) \label{eq:156ap}
    \end{align}
    with probability at least $1-\mathcal{O}(\frac{1}{K})$. This completes the proof of Lemma \ref{lem:5}.
	
	\section{Proof of Theorem \ref{th:3}} \label{D}
	Since Exp3-GR chooses the experts one by one for the exploration at the first $KM$ time instants, \eqref{eq:71ap} and \eqref{eq:37ap} hold true for Exp3-GR. In addition, for $t>KM$ we have
	\begin{align}
		\frac{W_{t+1}}{W_t} = \sum_{i=1}^{K}{ \frac{w_{i,t+1}}{W_t} } = \sum_{i=1}^{K}{ \frac{w_{i,t}}{W_t}\exp\left(-\eta \tilde{\ell}_{t}(v_i)\right) }. \label{eq:87ap}
	\end{align}
	According to \eqref{eq:21}, $\frac{w_{i,t}}{W_t}$ can be expressed as
	\begin{align}
		\frac{w_{i,t}}{W_t} = \frac{\pi_{i,t}-\frac{\eta}{|\mathcal{D}|}\mathcal{I}(v_i \in \mathcal{D})}{1-\eta}. \label{eq:88ap}
	\end{align}
	Therefore, \eqref{eq:88ap} can be rewritten as
	\begin{align}
		\frac{W_{t+1}}{W_t} = \sum_{i=1}^{K}{ \frac{\pi_{i,t}-\frac{\eta}{|\mathcal{D}|}\mathcal{I}(v_i \in \mathcal{D})}{1-\eta}\exp\left(-\eta \tilde{\ell}_{t}(v_i)\right) }. \label{eq:89ap}
	\end{align} 
	Furthermore, using the inequality $e^{-x} \le 1-x+\frac{1}{2}x^{2}, \forall x \ge 0$, we have
	\begin{align}
		\frac{W_{t+1}}{W_t} \le  \sum_{i=1}^{K}{  \frac{\pi_{i,t}-\frac{\eta}{|\mathcal{D}|}\mathcal{I}(v_i \in \mathcal{D})}{1-\eta}\left(1 -\eta \tilde{\ell}_{t}(v_i) + \frac{1}{2}(\eta \tilde{\ell}_{t}(v_i))^2\right) }. \label{eq:90ap}
	\end{align}
	Considering the inequality $1+x \le e^x$ and taking logarithm from both sides of  \eqref{eq:90ap}, we obtain
	\begin{align}
		\ln \frac{W_{t+1}}{W_t} \le \sum_{i=1}^{K}{  \frac{\pi_{i,t}-\frac{\eta}{|\mathcal{D}|}\mathcal{I}(v_i \in \mathcal{D})}{1-\eta}\left( -\eta \tilde{\ell}_{t}(v_i) + \frac{1}{2}(\eta \tilde{\ell}_{t}(v_i))^2\right) }. \label{eq:91ap}
	\end{align}
	Summing \eqref{eq:91ap} over $t$ from $t^{\prime} = KM+1$ to $T$, it can be written that
	\begin{align}
		\ln \frac{W_{T+1}}{W_{t^{\prime}}} \le \sum_{t=t^{\prime}}^{T}{ \sum_{i=1}^{K}{ \frac{\pi_{i,t}-\frac{\eta}{|\mathcal{D}|}\mathcal{I}(v_i \in \mathcal{D})}{1-\eta}\left( -\eta \tilde{\ell}_{t}(v_i) + \frac{1}{2}(\eta \tilde{\ell}_{t}(v_i))^2\right) } }. \label{eq:92ap}
	\end{align}
	In addition, $\ln \frac{W_{T+1}}{W_{t^{\prime}}}$ can be bounded from below as
	\begin{align}
		\ln \frac{W_{T+1}}{W_{t^{\prime}}} \ge \ln \frac{w_{i,T+1}}{W_{t^\prime}} = -\eta \sum_{t=t^{\prime}}^{T}{ \tilde{\ell}_{t}(v_i)} - \ln K. \label{eq:107ap}
	\end{align}
	Combining \eqref{eq:107ap} with \eqref{eq:92ap}, we find
	\begin{align}
		& \sum_{t=t^{\prime}}^{T}{ \sum_{i=1}^{K}{ \frac{\eta \pi_{i,t}}{1-\eta} \tilde{\ell}_{t}(v_i) } } -\eta \sum_{t=t^{\prime}}^{T}{ \tilde{\ell}_{t}(v_i)} \nonumber \\ \le & \ln K + \sum_{t=t^{\prime}}^{T}{\sum_{i \in \mathcal{D}}{ \frac{\eta^2}{|\mathcal{D}|(1-\eta)} \tilde{\ell}_{t}(v_i) }} \nonumber \\ & + \sum_{t=t^{\prime}}^{T}{\sum_{i=1}^{K}{ \eta^2 \frac{\pi_{i,t}-\frac{\eta}{|\mathcal{D}|}\mathcal{I}(v_i \in \mathcal{D})}{2(1-\eta)} \tilde{\ell}_{t}(v_i)^2 }}. \label{eq:93ap}
	\end{align}
Multiplying both sides of \eqref{eq:93ap} by $\frac{1- \eta }{\eta}$ as
	\begin{align}
	& \sum_{t=t^{\prime}}^{T}{ \sum_{i=1}^{K}{ \pi_{i,t} \tilde{\ell}_{t}(v_i) } } - \sum_{t=t^{\prime}}^{T}{ \tilde{\ell}_{t}(v_i)} \nonumber \\ \le & \frac{\ln K}{\eta} + \sum_{t=t^{\prime}}^{T}{\sum_{i \in \mathcal{D}}{ \frac{\eta}{|\mathcal{D}|} \tilde{\ell}_{t}(v_i) }} + \sum_{t=t^{\prime}}^{T}{\sum_{i=1}^{K}{ \frac{\eta}{2} (\pi_{i,t}-\frac{\eta}{|\mathcal{D}|}\mathcal{I}(v_i \in \mathcal{D})) \tilde{\ell}_{t}(v_i)^2 }}. \label{eq:94ap}
	\end{align} 
	According to \eqref{eq:22}, expected value of loss estimate $\tilde{\ell}_t(v_i)$ can be expressed as
	\begin{subequations} \label{eq:39ap}
	\begin{align}
	    \mathbb{E}_{t}[{\tilde{\ell}_{t}(v_i)}] & = \sum_{\forall j: v_j \in \mathcal{N}_{i,t}^{\text{in}}}{ \pi_{j,t}p_{ji} \mathbb{E}_t[Q_{i,t}]\ell_{t}(v_i) } = q_{i,t}\mathbb{E}_t[Q_{i,t}]\ell_{t}(v_i) \label{eq:39apa} \\
	    \mathbb{E}_{t}[{\tilde{\ell}_{t}(v_i)}^2] & = \sum_{\forall j: v_j \in \mathcal{N}_{i,t}^{\text{in}}}{ \pi_{j,t}p_{ji} \mathbb{E}_t[Q_{i,t}^2]\ell_{t}(v_i)^2 } = q_{i,t}\mathbb{E}_t[Q_{i,t}^2]\ell_{t}(v_i)^2. \label{eq:39apb}
	\end{align}
	\end{subequations}
	Note that the expected values depend on random variable $\{Z_{i,u}(t)\}_{u=1}^{M}$  in \eqref{eq:23}, where $P_{i,u}(t)$ and $Y_{ij,u}(t)$, $\forall i \in [K]$, $\forall (i,j) \in \mathcal{E}_t$  are independent Bernoulli random variables with parameters $\pi_{i,t}$ and $p_{ij}$, respectively. Therefore,  $\{Z_{i,u}(t)\}_{u=1}^{M}$ are also Bernoulli random variables with expected value
	\begin{align}
	\mathbb{E}_t[Z_{i,u}(t)] & = \mathbb{E}_t[\sum_{\forall j: v_j \in \mathcal{N}_{i,t}^{\text{in}}}{P_{j,u}(t)Y_{ji,u}(t)}] = \sum_{\forall j: v_j \in \mathcal{N}_{i,t}^{\text{in}}}{\mathbb{E}_t[P_{j,u}(t)]\mathbb{E}_t[Y_{ji,u}(t)]} \nonumber \\ & = \sum_{\forall j: v_j \in \mathcal{N}_{i,t}^{\text{in}}}{\pi_{j,t}p_{ji}} = q_{i,t}. \label{eq:41ap}
	\end{align}
	In other words, $Z_{i,u}(t)$ is a Bernoulli random variable whose value is $1$ with probability $q_{i,t}$. The expected value of $Q_{i,t}$ and $Q_{i,t}^2$ can henceforth be written as
	\begin{subequations} \label{eq:40ap}
	    \begin{align}
	    \mathbb{E}_t[Q_{i,t}] & = \sum_{u=1}^{M}{ uq_{i,t}(1-q_{i,t})^{u-1} } + M(1-q_{i,t})^M  \nonumber \\ & = \frac{1-(Mq_{i,t}+1)(1-q_{i,t})^M}{q_{i,t}} + M(1-q_{i,t})^M \nonumber \\ & = \frac{1-(1-q_{i,t})^M}{q_{i,t}} \label{eq:40apa} \\
	    \mathbb{E}_t[Q_{i,t}^2] & = \sum_{u=1}^{M}{ u^2q_{i,t}(1-q_{i,t})^{u-1} } + M^2(1-q_{i,t})^M \nonumber \\ & = \frac{2 - 2(1-q_{i,t}^{M+2})}{q_{i,t}^2} - \frac{1+(2M+3)(1-q_{i,t})^{M+1}}{q_{i,t}} \nonumber \\ & \hspace{5mm} - (M+1)^2(1-q_{i,t})^M + M^2(1-q_{i,t})^M \nonumber \\ & = \frac{2- 2(1-q_{i,t}^{M+2})}{q_{i,t}^2} - \frac{1+(2M+3)(1-q_{i,t})^{M+1}}{q_{i,t}} \nonumber \\ & \hspace{5mm} - (2M+1)(1-q_{i,t})^M. \label{eq:40apb}
	    \end{align}
	\end{subequations}
	Combining \eqref{eq:39ap} with \eqref{eq:40ap}, we arrive at
	\begin{subequations} \label{eq:42ap}
	\begin{align}
	    \mathbb{E}_{t}[{\tilde{\ell}_{t}(v_i)}] & = q_{i,t}\frac{1-(1-q_{i,t})^M}{q_{i,t}}\ell_{t}(v_i) \nonumber \\ &= \left(1-(1-q_{i,t})^M\right)\ell_{t}(v_i) \le \ell_{t}(v_i) \label{eq:42apa} \\
	    \mathbb{E}_{t}[{\tilde{\ell}_{t}(v_i)}^2] &= \left(\frac{2- 2(1-q_{i,t}^{M+2})}{q_{i,t}} - 1+(2M+3)(1-q_{i,t})^{M+1}\right)\ell_{t}(v_i)^2 \nonumber \\ & \hspace{5mm} - q_{i,t}(2M+1)(1-q_{i,t})^M\ell_{t}(v_i)^2 \nonumber \\ & \le \frac{2}{q_{i,t}}. \label{eq:42apb}
	\end{align}
	\end{subequations}
	Combining \eqref{eq:94ap} and \eqref{eq:42ap}, it can be concluded that
	\begin{align}
		& \sum_{t=t^{\prime}}^{T}{ \sum_{i=1}^{K}{ \pi_{i,t} \ell_{t}(v_i) } } - \sum_{t=t^{\prime}}^{T}{ \sum_{i=1}^{K}{ \pi_{i,t} (1-q_{i,t})^M \ell_{t}(v_i) } } - \sum_{t=t^{\prime}}^{T}{ \ell_{t}(v_i)} \nonumber \\ \le & \frac{\ln K}{\eta} + \sum_{t=t^{\prime}}^{T}{\sum_{i \in \mathcal{D}}{ \frac{\eta}{|\mathcal{D}|} \ell_{t}(v_i) }} + \sum_{t=t^{\prime}}^{T}{\sum_{i=1}^{K}{ \frac{\eta}{2} (\pi_{i,t}-\frac{\eta}{|\mathcal{D}|}\mathcal{I}(v_i \in \mathcal{D})) \frac{2}{q_{i,t}} }}.   \label{eq:43ap} 
	\end{align}
	According to (a1) $\ell_{t}(v_i) \le 1$ and using the fact that $\frac{2}{q_{i,t}} \ge 2$, \eqref{eq:43ap} can be further bounded by
	\begin{align}
	& \sum_{t=t^{\prime}}^{T}{ \sum_{i=1}^{K}{ \pi_{i,t} \ell_{t}(v_i) } } - \sum_{t=t^{\prime}}^{T}{ \ell_{t}(v_i)} \nonumber \\ \le & \frac{\ln K}{\eta} + \sum_{t=t^{\prime}}^{T}{ (1-q_{i,t})^M } + \sum_{t=t^{\prime}}^{T}{\sum_{i \in \mathcal{D}}{ \frac{\eta-\eta^2}{|\mathcal{D}|}  }} + \sum_{t=t^{\prime}}^{T}{\sum_{i=1}^{K}{ \eta \frac{\pi_{i,t}}{q_{i,t}} }}\nonumber\\
	= & \frac{\ln K}{\eta} + \sum_{t=t^{\prime}}^{T}{ (1-q_{i,t})^M } + \eta(1-\eta)(T-KM) + \eta \sum_{t=t^{\prime}}^{T}{\sum_{i=1}^{K}{ \frac{\pi_{i,t}}{q_{i,t}} } }. \label{eq:45ap}
	\end{align}
	Note that when $t > t^\prime$, we have $\mathbb{E}_{t}[\ell_t(v_{I_t})] = \sum_{i=1}^{K}{ \pi_{i,t} \ell_{t}(v_i) }$. Combining \eqref{eq:37ap} with \eqref{eq:45ap} leads to
	\begin{align}
	& \sum_{t=1}^{T}{ \mathbb{E}_{t}[\ell_t(v_{I_t})] } - \sum_{t=1}^{T}{ \ell_{t}(v_i)} \nonumber \\ \le & \frac{\ln K}{\eta} + (K-1)M + \sum_{t=t^{\prime}}^{T}{ (1-q_{i,t})^M } \nonumber \\ & + \eta(1-\eta)(T-KM) + \eta \sum_{t=t^{\prime}}^{T}{\sum_{i=1}^{K}{ \frac{\pi_{i,t}}{q_{i,t}} } } \label{eq:46ap}
	\end{align}
	which completes the proof of Theorem \ref{th:3}.
	
	\section{Proof of Corollary \ref{cor:2}} \label{F}

According to (a2), if $(i,j) \in \mathcal{E}$, the learner observes the loss of the $j$-th expert when it chooses the $i$-th expert with probability at least $\epsilon$. Recalling \eqref{eq:21} it can be inferred that $\pi_{i,t}> \eta/|\mathcal{D}|$, $\forall i \in \mathcal{D}$. Combining the fact that for each $v_i \in \mathcal{V}$ there is at least one edge from $\mathcal{D}$ to $v_i$, $\forall i \in [K]$ with \eqref{eq:6},  $q_{i,t}$ can be bounded below as 
\begin{align}
q_{i,t} > \frac{\eta \epsilon}{|\mathcal{D}|}. \label{eq:125ap}
\end{align}
Combining the condition
\begin{align}
M \ge \frac{|\mathcal{D}| \ln T}{2\eta \epsilon}\label{eq:cond:ap}
\end{align}
with \eqref{eq:125ap}, we have $
	    Mq_{i,t} \ge \frac{1}{2} \ln T 
$
	which leads to 
	\begin{align}
	    e^{-Mq_{i,t}} \le \frac{1}{\sqrt{T}}.\label{eq:108:ap}
	\end{align}
Combining \eqref{eq:108:ap} with the fact $1+x \le e^x$, we have
	\begin{align}
	    (1-q_{i,t})^M \le e^{-Mq_{i,t}} \le \frac{1}{\sqrt{T}}. \label{eq:104ap}
	\end{align}	
Hence, the third term in \eqref{eq:20}, i.e.,  $\sum_{t=t^{\prime}}^{T}{ (1-q_{i,t})^M }$  can be bounded by $\mathcal{O}(\sqrt{T})$.


Furthermore, consider the case where we have
\begin{align}
    \eta = \mathcal{O}(\sqrt{\frac{\ln K}{T}}). \label{eq:115ap}
\end{align}
Therefore, taking into account that greedy set cover algorithm is used to determine the dominating set $\mathcal{D}$, based on \eqref{eq:134ap} it can be obtained that
\begin{align}
    M = \mathcal{O}(\frac{\alpha(\mathcal{G})}{\epsilon}\ln T \sqrt{T \ln K}), \label{eq:124ap}
\end{align}
satisfies the condition in \eqref{eq:cond:ap}. Hence, the expected regret of Exp3-GR satisfies
\begin{align}
    & \sum_{t=1}^{T}{ \mathbb{E}_{t}[\ell_t(v_{I_t})] } - \min_{v_i \in \mathcal{V}}{ \sum_{t=1}^{T}{\ell_t(v_i)} } \nonumber \\ \le &  \mathcal{O}\left(\frac{\alpha(\mathcal{G})}{\epsilon}\sqrt{\ln K}(\ln(KT)+K \ln T)\sqrt{T}\right). \label{eq:117ap}
\end{align}

\section{Proof of Lemma \ref{lem:6}} \label{I}
In this section, the doubling trick is employed to choose $\eta$ and $M$ when the learner does not know the time horizon $T$, beforehand. At time instant $t$, when $2^b < t \le 2^{b+1}$, for $\eta_b$ and $M_b$ the following values are chosen
\begin{subequations} \label{eq:157ap}
\begin{align}
    \eta_b &= \sqrt{\frac{\ln K}{2^{b+1}}} \label{eq:157apa} \\
    M_b &= \left \lceil{\frac{(b+1)\sqrt{2^{b-1}}|\mathcal{D}|\ln 2}{\epsilon \sqrt{\ln K}}} \right \rceil. \label{eq:157apb}
\end{align}
\end{subequations}
When $t > 2^{b+1}$ holds true, the algorithm restarts with $\eta_{b+1}$ and $M_{b+1}$. The algorithm starts with $b=0$. At each restart, the algorithm chooses the experts one by one for exploration until the condition that each expert is chosen at least $M_b$ times is met. Then, the learner uses the last $M_b$ observed samples from each expert to perform geometric resampling. In this case, for each $b$, Exp3-GR satisfies
\begin{align}
    & \sum_{t=2^b+1}^{T_b}{ \mathbb{E}_{t}[\ell_t(v_{I_t})] } - \min_{v_i \in \mathcal{V}}{ \sum_{t=2^b+1}^{T_b}{\ell_t(v_i)} } \nonumber \\  \le & \frac{\ln K}{\eta_b} + (K-1)(M_b-M_{b-1}) + \sum_{t=t_b^\prime}^{T_b}{ (1-q_{i,t})^{M_b} } \nonumber \\ &  + \eta_b(1-\eta_b)(T_b-2^b-K(M_b-M_{b-1})) + \eta_b \sum_{t=t_b^\prime}^{T_b}{\sum_{i=1}^{K}{ \frac{\pi_{i,t}}{q_{i,t}} } } \label{eq:158ap}
\end{align}
where $T_b$ denote the greatest time instant which satisfies $2^b < T_b \le 2^{b+1}$ and $t_b^\prime$ can be expressed as in \eqref{eq:149ap}. Note that when $b=0$, we have $M_{b-1}=0$. Taking into account that the maximum loss at each restart is $1$, summing \eqref{eq:158ap} over all possible values for $b$ obtains
\begin{align}
    & \sum_{t=1}^{T}{ \mathbb{E}_{t}[\ell_t(v_{I_t})] } - \min_{v_i \in \mathcal{V}}{ \sum_{t=1}^{T}{\ell_t(v_i)} } \nonumber \\ \le & \left \lceil{\log_2 T} \right \rceil + \sum_{b=0}^{\left \lfloor{\log_2 T} \right \rfloor}{\frac{\ln K}{\eta_b}} + (K-1)M + \sum_{b=0}^{\left \lfloor{\log_2 T} \right \rfloor}{\sum_{t=t_b^\prime}^{T_b}{ (1-q_{i,t})^{M_b} }} \nonumber \\ & + \sum_{b=0}^{\left \lfloor{\log_2 T} \right \rfloor}{\eta_b(1-\eta_b)(T_b-2^b-K(M_b-M_{b-1}))} \nonumber \\ & + \sum_{b=0}^{\left \lfloor{\log_2 T} \right \rfloor}{\eta_b \sum_{t=t_b^\prime}^{T_b}{\sum_{i=1}^{K}{ \frac{\pi_{i,t}}{q_{i,t}} } } } \label{eq:159ap}
\end{align}
where $M$ is the number of samples for each expert when $b=\left \lfloor{\log_2 T} \right \rfloor$ which are used for geometric resampling during the learning task. According to \eqref{eq:157apb} and based on the fact that $\mathcal{D}$ is obtained using the greedy set cover algorithm, it can be written that
\begin{align}
    M = \mathcal{O}(\frac{\alpha(\mathcal{G})}{\epsilon}\ln T \sqrt{T\ln K}). \label{eq:160ap}
\end{align}
Furthermore, for each $b$, the inequality $q_{i,t}>\frac{\eta_b \epsilon}{|\mathcal{D}|}$ holds. Therefore, according to \eqref{eq:157ap}, we can write $M_b q_{i,t}>\frac{b+1}{2}\ln 2$. Thus, it can be concluded that
\begin{align}
    (1-q_{i,t})^{M_b} \le e^{-M_b q_{i,t}} < \frac{1}{\sqrt{2^{b+1}}}. \label{eq:161ap}
\end{align}
Using \eqref{eq:161ap}, we obtain
\begin{align}
    \sum_{b=0}^{\left \lfloor{\log_2 T} \right \rfloor}{\sum_{t=t_b^\prime}^{T_b}{ (1-q_{i,t})^{M_b} }} < \sum_{b=0}^{\left \lfloor{\log_2 T} \right \rfloor}{\frac{T_b - 2^b}{\sqrt{2^{b+1}}}} \le \sum_{b=0}^{\left \lfloor{\log_2 T} \right \rfloor}{\sqrt{2^{b-1}}} \le \frac{\sqrt{2T}-1}{2-\sqrt{2}}. \label{eq:162ap}
\end{align}
In addition, based on the Lemma \ref{lem:3}, it can be written that
\begin{align}
    \sum_{b=0}^{\left \lfloor{\log_2 T} \right \rfloor}{\eta_b \sum_{t=t_b^\prime}^{T_b}{\sum_{i=1}^{K}{ \frac{\pi_{i,t}}{q_{i,t}} } } } & \le \sum_{b=0}^{\left \lfloor{\log_2 T} \right \rfloor}{\sqrt{\frac{\ln K}{2^{b+1}}}(T_b-2^b)\mathcal{O}\left(\frac{\alpha(\mathcal{G})}{\epsilon}\ln(KT)\right)} \nonumber \\ & \le \left \lceil{\log_2 T} \right \rceil \sqrt{2^{b-1}\ln K}\mathcal{O}\left(\frac{\alpha(\mathcal{G})}{\epsilon}\ln(KT)\right) \nonumber \\ & = \mathcal{O}\left(\frac{\alpha(\mathcal{G})}{\epsilon}(\ln T)\ln(KT)\sqrt{T\ln K}\right). \label{eq:164ap}
\end{align}
Therefore, combining \eqref{eq:159ap} with \eqref{eq:160ap}, \eqref{eq:162ap} and \eqref{eq:164ap}, it can be inferred that Exp3-GR satisfies
\begin{align}
    \sum_{t=1}^{T}{ \mathbb{E}_{t}[\ell_t(v_{I_t})] } - \min_{v_i \in \mathcal{V}}{ \sum_{t=1}^{T}{\ell_t(v_i)} } \le \mathcal{O}\left(\frac{\alpha(\mathcal{G})\ln T}{\epsilon}(\ln(KT)+K)\sqrt{T\ln K} \right) \label{eq:165ap}
\end{align}
which completes the proof of Lemma \ref{lem:6}.

\end{document}